\documentclass[lettersize,journal]{IEEEtran}
\usepackage{amsmath,amsfonts}
\usepackage{algorithmic}
\usepackage{algorithm}
\usepackage{array}
\usepackage{textcomp}
\usepackage{stfloats}
\usepackage{url}
\usepackage{color}
\usepackage{subfigure}
\usepackage{makecell}
\usepackage{multirow}
\usepackage{verbatim}
\usepackage{graphicx}
\usepackage{bbding}
\usepackage{cite}
\usepackage{enumitem}
\usepackage{amsmath,amssymb,amsthm}
\usepackage{gensymb}
\usepackage{dsfont}
\usepackage{threeparttable}
\usepackage{url}
\usepackage[colorlinks,
            allcolors=black,
            linkcolor=black,
            anchorcolor=black,
            citecolor=black]{hyperref}
            
\newtheorem{theorem}{Theorem}

\newtheorem{lemma}{Lemma}

\usepackage{accents}
\makeatletter
\def\widebar{\accentset{{\cc@style\underline{\mskip15mu}}}}
\makeatother

\usepackage{nomencl}
\makenomenclature

\begin{document}

\title{Grasp, See, and Place: Efficient Unknown Object Rearrangement with Policy Structure Prior}

\author{Kechun Xu, Zhongxiang Zhou, Jun Wu, Haojian Lu, Rong Xiong, Yue Wang
\thanks{{This work was supported by the National Key R$\&$D Program of China (2023YFB4705001) and the National Nature Science Foundation of China under Grant 62173293.} Kechun Xu, Zhongxiang Zhou, Jun Wu, Haojian Lu, Rong Xiong, Yue Wang are with Zhejiang University, Hangzhou, China. Corresponding author, {\tt\small wangyue@iipc.zju.edu.cn}. Project page: {\href{https://xukechun.github.io/papers/GSP/}{https://xukechun.github.io/papers/GSP}}.}
}



\maketitle

\begin{abstract}
We focus on the task of unknown object rearrangement, where a robot is supposed to re-configure the objects into a desired goal configuration specified by an RGB-D image. Recent works explore unknown object rearrangement systems by incorporating learning-based perception modules. However, they are sensitive to perception error, and pay less attention to task-level performance. In this paper, we aim to develop an effective system for unknown object rearrangement amidst perception noise. We theoretically reveal that the noisy perception impacts grasp and place in a decoupled way, and show such a decoupled structure is {valuable} to improve task optimality. We propose GSP, a dual-loop system with the decoupled structure as prior. For the inner loop, we learn {a see} policy for self-confident {in-hand object matching}. For the outer loop, we learn a grasp policy aware of object matching and grasp capability guided by task-level rewards. We leverage the foundation model CLIP for object matching, policy learning and self-termination. A series of experiments indicate that GSP can conduct unknown object rearrangement with higher completion rates and {fewer} steps.
\end{abstract}

\begin{IEEEkeywords}
Unknown Object Rearrangement, Active Perception, Object Manipulation, Task and Motion Planning.
\end{IEEEkeywords}

{
\printnomenclature
\nomenclature[03]{$o_c^i$}{Detected object crop in the current scene.}
\nomenclature[04]{$o_g^j$}{Detected object crop in the goal scene.}
\nomenclature[05]{$o_h$}{Detected crop of the in-hand object.}
\nomenclature[07]{$\mathcal{O}_c$}{Set of $M$ detected object crops in the current scene.}
\nomenclature[08]{$\mathcal{O}_g$}{Set of $N$ detected object crops in the goal scene.}
\nomenclature[09]{$\mathcal{O}_{c,1}$}{Set of objects that can be directly moved to goals.}
\nomenclature[10]{$\mathcal{O}_{c,2}$}{Set of objects whose goals are occupied by others.}
\nomenclature[13]{$\pi$}{Rearrangement policy.}
\nomenclature[14]{$\pi_{\mathcal{G}}$}{Grasp policy.}
\nomenclature[15]{$\pi_{\mathcal{P}}$}{Place policy.}
\nomenclature[17]{$\pi_{\mathcal{S}}$}{See policy.}
\nomenclature[18]{$a_{\mathcal{G}}$}{Grasp action.}
\nomenclature[19]{$a_{\mathcal{P}}$}{Place action.}
\nomenclature[20]{$a_{\mathcal{S}}$}{See action.}
\nomenclature[22]{$\mathcal{M}_c^i(j)$}{Object matching distribution of the current object $o_c^i$.}
}

\section{Introduction}
\IEEEPARstart{O}{bject} rearrangement is essential for an everyday robot to re-configure the objects into a desired goal configuration, which is an important problem in embodied AI research~\cite{batra2020rearrangement}. Traditionally, the current and goal object configurations are represented as symbolic states {\it e.g.} pose and identity. Guided by the task completion and efficiency, the rearrangement problem is generally solved using task and motion planning (TAMP)~\cite{krontiris2015dealing,king2016rearrangement,huang2019large, song2020multi,gao2021minimizing,gao2022fast,wada2022reorientbot,xu2022efficient,tian2022sampling,king2017unobservable,nam2019planning,garrett2020online,driess2020deep,liu2021ocrtoc,zhu2021hierarchical} for optimal task-level performance. Depending on the symbolic state, these methods assume known object models, thus encountering difficulty in open real-world scenarios with unknown objects. 

\begin{figure}[t]
  \centering
  \includegraphics[width=\linewidth]{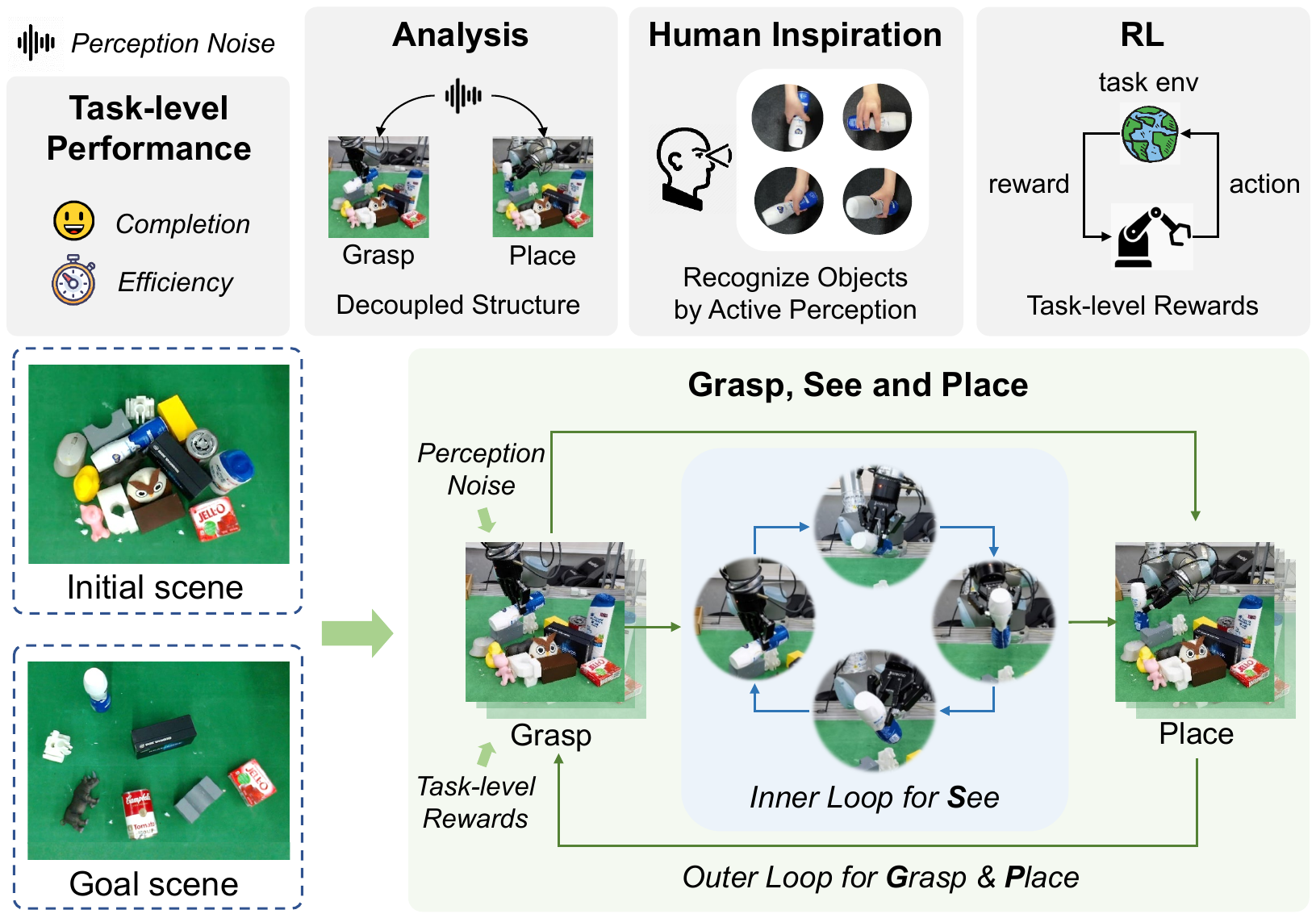}
  \vspace{-0.5cm}
  \caption{{{\bf G}rasp, {\bf S}ee, and {\bf P}lace. The robot is given the initial and goal scenes for the task of object rearrangement. Aiming at improving task-level performance with perception noise, we first derive the decoupled structure by analysis. Guided by the decoupled prior, we incorporate human behavior and task-level rewards into the general framework of GSP. In general, GSP contains two loops: the inner loop actively sees the grasped object for high self-confident matching, and the outer loop conducts the grasp and place planning.}}
  \label{fig:teaser}
  \vspace{-0.6cm}
\end{figure}

Recent works explore solutions for unknown object rearrangement with the goal object configuration given as an RGB-D image~\cite{labbe2020monte,xu2021learning,zeng2021transporter,qureshi2021nerp,goyal2022ifor,tang2023selective}. Efforts have been made to representation learning of current and goal object configurations by leveraging learning-based perception models. The pioneering work NeRP~\cite{qureshi2021nerp} builds a scene graph by matching the objects in current and goal configurations with 2-DoF {pose transformations} {\it i.e.} in-plane translations. Based on the graph, grasp and place action are learned from demonstration to complete the task. Selective Rearrangement~\cite{tang2023selective} further extends the graph-based representation to address settings with clutter and selectivity {\it i.e.} rearranging a subset of objects. The policy is guided by the graph editing distance. IFOR~\cite{goyal2022ifor} employs a pixel-level optical flow to represent the difference between current and goal object configurations, which allows for rearranging unknown objects with 3-DoF {pose transformations} {\it i.e.} in-plane translations and in-plane rotations. By minimizing the flow, the robot places the objects selected according to handcrafted rules. These fruitful progresses largely narrow down the gap of perception side from known to unknown object rearrangement. 

In contrast, looking into the planning side of unknown object rearrangement, we find that existing systems derive actions by heuristic rules or supervision, which pay less attention to the optimality of task-level performance. Taking a simple case of two objects that can be directly rearranged to the goal, their moving order does not affect the task-level performance. Hence, regarding either order as the unique ground truth for supervision brings bias into policy learning. Perception noise makes the situation even worse, as the heuristics may be built on incorrect perception results. Considering the classic task-level optimal policy design for object rearrangement with ideal perception, we raise a question: \textit{Given the noisy perception results, is there a policy that can optimize the task-level performance of unknown object rearrangement?}

To optimize task-level performance, reinforcement learning (RL) is a useful tool. However, it is challenging to directly learn the policy with RL for the long-horizon object rearrangement. In this paper, we delve into the structure of the policy by beginning with {theoretical} analysis on object rearrangement with ideal perception. We show that, to minimize the total steps, it is an optimal policy to grasp objects whose goals are non-occupied and place them to their goals, and resolve the objects whose goals are circularly occupied with the aid of buffers. Additionally, we find that the optimal grasp is multimodal, verifying the difficulty of the per-step supervision of action. When extending to noisy perception, we derive that the noise impacts the grasp and place in a \textit{decoupled} way. We furthermore show such a decoupled structure is {valuable} to improve the task-level optimality. Besides, the perception noise makes the grasp supervision even harder.

Guided by the insights, we propose an RL-based rearrangement policy with the decoupled structure as prior, namely GSP ({\bf G}rasp, {\bf S}ee and {\bf P}lace). {GSP contains two loops: an {outer loop} for grasp and place with an {inner loop} for see~(Fig.~\ref{fig:teaser}).} Thanks to the decoupled prior, we can boost the perception performance of placing individually by an inner loop. Inspired by the human behavior of looking at an object up and down to confirm its identity~\cite{bajcsy1988active,aloimonos1988active,chen2011active}), we propose to independently learn a see policy by actively rotating the in-hand object to improve {the in-hand object matching} for task optimality. 
For the outer loop, we propose to learn a policy conditioned on the uncertain object matching between current and goal images as well as grasp capability, guided by task-level rewards. Such a dual-loop reinforcement learning structure relieves the sparse reward in long-horizon tasks and focuses on task-level performance instead of per-step supervision. To model the object matching, we leverage the foundation model CLIP~\cite{radford2021learning} for the similarity measure between objects in current and goal images, which is utilized for policy learning and self-termination. Thanks to the pre-training of CLIP on large-scale data, our system demonstrates zero-shot generalization to unseen objects. We evaluate the system on cluttered and selective unknown object rearrangement tasks with {6-DoF pose transformations}. Empowered by the structure prior and the foundation model, GSP demonstrates higher task completion rates using {fewer} steps in both simulation and real-world experiments. To summarize, the main contributions are:
\begin{itemize}
    \item We theoretically reveal {that} noisy perception impacts grasp and place in a decoupled way, and show that it is {valuable} to improve {the in-hand object matching}.
    \item We propose a dual-loop policy {\it i.e.} GSP with the decoupled structure as the prior for better efficiency and dealing with unknown objects.
    \item For the inner loop, we learn {a see} policy for self-confident object matching, which improves {the in-hand object matching} for task optimality. 
    \item For the outer loop, we learn a policy to select a grasp action aware of object matching and grasping capability, guided by task-level rewards.
    \item We leverage the foundation model CLIP for object matching, policy learning and self-termination, which links the two loops for better policy performance. 
    \item We evaluate our method with scenarios on seen and unseen objects in both simulated and real-world settings. The results validate the effectiveness and generalization.
\end{itemize}

\section{Related Works}
\subsection{Tabletop Object Rearrangement}
Tabletop object rearrangement is a topic that has been explored for decades, and is considered as a typical challenge in embodied AI~\cite{batra2020rearrangement}. In this task, the robot is supposed to re-configure the objects into a desired goal configuration. Conventional works for object rearrangement~\cite{krontiris2015dealing,king2016rearrangement,huang2019large,song2020multi,gao2021minimizing,gao2022fast,wada2022reorientbot,xu2022efficient,tian2022sampling,wang2021uniform,wang2022efficient} are studied in the context of task and motion planning (TAMP), which are generally under the assumption of fully-observable states and dynamics. While some works extend these methods to fully or partially occlusion cases through learning-based methods~\cite{king2017unobservable,nam2019planning,garrett2020online,driess2020deep,liu2021ocrtoc,zhu2021hierarchical,zeng2018semantic}, they still rely on known object models for accurate pose estimation to build symbolic states. As a result, it is unfeasible to deploy these methods in open real-world environments where there are novel objects without known models, significantly limiting the applications. 

Recent works leverage the learning-based models of perception and planning to explore solutions for unknown object rearrangement with the goal configuration specified by an RGB-D image~\cite{labbe2020monte,xu2021learning,xu2021efficient,zeng2021transporter,qureshi2021nerp,goyal2022ifor,tang2023selective}. DSR-Net~\cite{xu2021learning} learns a 3D dynamic scene representation for object pushing. Transporter Networks~\cite{zeng2021transporter} acquires rearrangement skills from demonstration with objects of simple shapes. The most relevant works to ours are~\cite{qureshi2021nerp,goyal2022ifor,tang2023selective}, which develop rearrangement systems for everyday objects with unknown models. NeRP~\cite{qureshi2021nerp} builds a scene graph by segmenting all objects in current and goal scenes with unseen object segmentation and matching objects with image feature similarity. Then, neural networks are learned for object and action planning from demonstration data. However, it only considers 2-DoF in-plane translations of objects, which limits its applications. IFOR~\cite{goyal2022ifor} shows the capability of unknown object rearrangement with 3-DoF {pose transformations} ({\it i.e.} in-plane translations and in-plane rotations) by applying a pixel-level optical flow to represent the difference between current and goal scenes. By minimizing the flow, the robot places the objects selected according to handcrafted rules. Selective Rearrangement~\cite{tang2023selective} further considers settings with clutter and selectivity ({\it i.e.} rearranging a subset of objects). Analogous to NeRP~\cite{qureshi2021nerp}, Selective Rearrangement~\cite{tang2023selective} obtains object masks by unseen object segmentation, and conducts object matching to build the scene graph. Then the policy is guided by graph editing distance. These methods largely narrow down the gap of perception side from known to unknown object rearrangement. However, from the planning side, they derive actions by heuristic rules or supervisions, which pay less attention to task-level performance. Additionally, they may be sensitive to the error in perception results. If objects are in clutter or required to be rearranged with large {pose transformations}, incorrect object matching may be inevitable. Other efforts to rearrange unknown objects with {6-DoF pose transformations}~\cite{cheng2022learning,simeonov2022neural,simeonov2023se,chun2023local} focus on object reorientation for the goal relation between two objects, and only generalize across objects of similar categories. 

In this work, we set to consider the task-level performance of unknown object rearrangement with noisy perception results ({\it i.e.} uncertainty of object matching). We learn {a see} policy to improve object matching, and jointly consider object matching and grasping capability for grasp planning, which is guided by task-level rewards.
In parallel to our work, other efforts solve rearrangement problems with language instructions~\cite{shridhar2022cliport,liu2022structformer,liu2022structdiffusion,paxton2022predicting,goodwin2022semantically,xu2023joint,xu2023object}.

\subsection{Active Perception}
Active perception is a broad concept describing situations where a robot adopts strategies to change the sensor configuration ({\it e.g.} camera viewpoint) to get better information~\cite{bajcsy1988active,aloimonos1988active,chen2011active}. There are various applications of active perception in object recognition, segmentation, navigation, and manipulation~\cite{eidenberger2010active,bohg2017interactive,ren2019domain,calli2018active,fu2019active,morrison2019multi,sun2019active,saito2021select,cheng2018reinforcement,jangir2022look,zeng2022robotic,ren2023robot}. 
\cite{calli2018active} proposes a model-free viewpoint optimization method to improve the performance of object recognition and grasping synthesis. \cite{fu2019active,morrison2019multi} plans the next-best-view by information gain approaches to reduce the grasp uncertainty for better object picking in clutter or occlusion. \cite{saito2021select} utilizes active perception to recognize characteristics of the target object, so as to select a tool for better manipulation. \cite{cheng2018reinforcement,jangir2022look} learns next-best-view policy through reinforcement learning by designing task-based rewards. Other works integrate object or environment interaction for exploration (referred to as Interactive Perception~\cite{bohg2017interactive}). \cite{danielczuk2019mechanical,novkovic2020object,bejjani2021occlusion,miao2022safe,huang2022mechanical} conduct mechanical search or retrieval of goal objects in clutter with nonprehensile actions, and \cite{xu2021learning} builds continuous and complete scene representation by planar pushing. 

In this work, we leverage the idea of active perception for better object matching between the in-hand object and objects in the goal scene. Specifically, we learn {a see} policy that reorients the in-hand object based on matching results, which aims to achieve highly self-confident object matching.

\begin{figure*}[t]
  \centering
  \includegraphics[width=\linewidth]{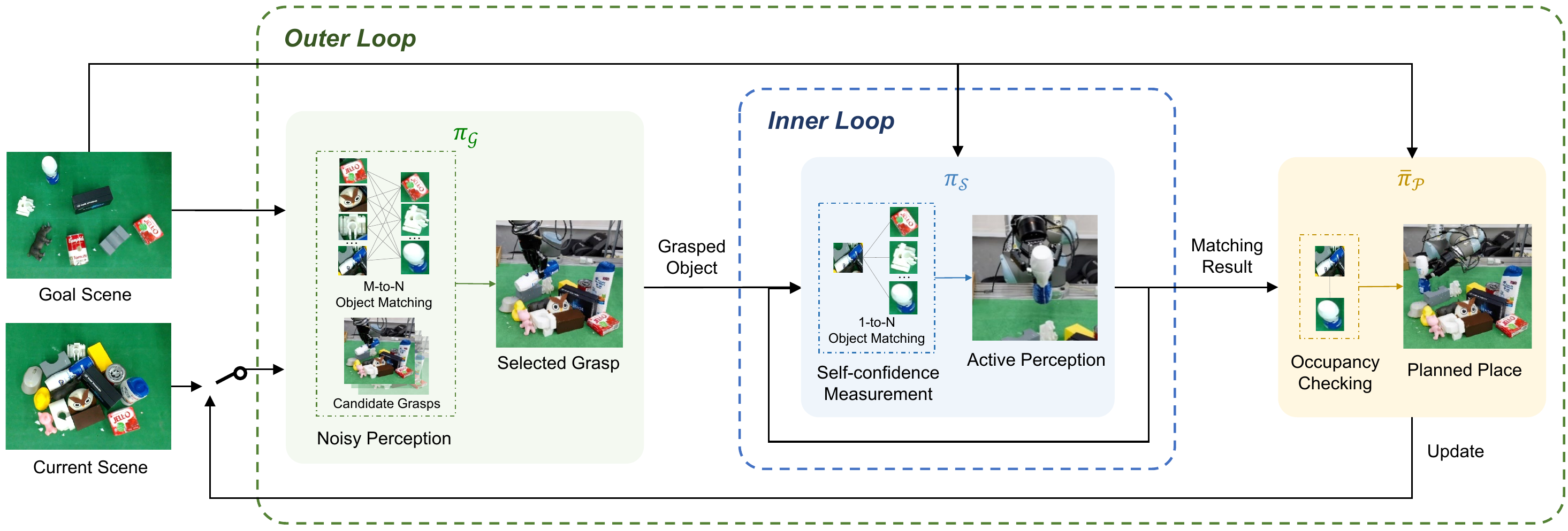}
  \vspace{-0.6cm}
  \caption{{{\bf System Overview}. Given the RGB-D images of the current and goal scenes, the grasp policy jointly considers object matching and candidate grasps to determine a selected grasp pose. After picking up an object, object matching is conducted between the grasped object and the goal objects. If the matching is self-confident, the object is rearranged to the planned place pose based on occupancy checking. Otherwise, active {perception} is triggered to predict the delta orientation of the end effector. Then the robot rotates the in-hand object to a new view until a confident matching is achieved. Overall, our method decomposes the object rearrangement process into two loops: an inner loop for {see} and an outer loop for grasp and place planning.}
  }
  \label{fig:overview}
  \vspace{-0.6cm}
\end{figure*}

\subsection{Embodied AI}
Our research lies in the broad domain of embodied AI, which concerns intelligent systems with a physical or virtual embodiment. In recent years, embodied AI has evolved from primarily focusing on navigation tasks~\cite{das2018embodied,deitke2020robothor,savva2019habitat,shridhar2020alfred,xia2018gibson} to now incorporating physical manipulation tasks~\cite{fan2018surreal,james2020rlbench, yu2020meta,ehsani2021manipulathor}. Lately, the problem of object rearrangement has been identified as a representative challenge for evaluating embodied AI~\cite{batra2020rearrangement}, and our work contributes to advancing this specific area. Moreover, the tabletop object rearrangement task in this work is more challenging than that defined in \cite{batra2020rearrangement} by considering task settings without any object models. Unlike the majority of previous studies in embodied AI that solely rely on simulated evaluations, we also assess the performance of our approach in real-world settings. {There are also efforts that focus on the sim2real problem with techniques like domain randomization~\cite{james2017transferring,james2019sim,tobin2017domain}. In this work, we employ pre-trained foundation models to process visual information, thus narrowing the sim2real gap.}

\section{Policy Structure and System Design}
In this paper, we aim to optimize the task-level performance of object rearrangement. One intuitive tool is reinforcement learning~(RL). However, it is challenging to directly learn the policy with RL for the long-horizon object rearrangement. Hence, we first delve into the structure of the policy, and then introduce our system guided by the structure.

{
\subsection{Problem Statement}
Consider a tabletop object rearrangement problem with the current and goal scenes specified as RGB-D images $I_c$ and $I_g$. $M$ detected object crops $\mathcal{O}_c=\{o_c^i\}_{i=1,...,M}$ in the current scene are supposed to be re-configured to $N$ detected object crops $\mathcal{O}_g=\{o_g^j\}_{j=1,...,N}$ in the goal scene. A current object $o_c^i$ that does not match any goal object, namely a non-goal object, is supposed to be placed outside the workspace, indicating $\mathcal{O}_g\subseteq\mathcal{O}_c$. Denote the object matching distribution between the current object $o_c^i$ and $\mathcal{O}_g$ as $\mathcal{M}_c^i(j)$. $\mathcal{M}_c^i(j)$ is a categorical distribution. Denote $\zeta_g$ as a threshold, if $\operatorname{max}_j \mathcal{M}_c^i(j)\geq \zeta_g$, $(0\textless \zeta_g \textless 1)$, then the matched goal index of $o_c^i$ is $j_i=\operatorname{argmax}_j \mathcal{M}_c^i(j)$. Otherwise, $o_c^i$ is regarded as a non-goal object. We consider grasp and place to move objects one by one. 

To complete the task, a rearrangement policy $\pi$ should determine the object moving order based on the perception results. For task-level efficiency, the optimal rearrangement policy is to minimize the total pick-n-place steps to achieve the goal configuration.

{\bf Assumptions.} We follow the assumptions in \cite{qureshi2021nerp,goyal2022ifor,tang2023selective}. First, there is no occlusion or stack in the goal scene, and the object localization of the goal scene is correct. This is because occlusion or stack hinders the goal representation to a large extent, limiting the upper bound of task-level performance for all policies. {Second, we assume that there are detected object bounding boxes and feasible predicted grasps in the current scene. Based on these assumptions, the camera angles of current and goal scenes are not necessarily the same.} Lastly, each object can be moved to its goal pose with one step of pick-n-place when the goal is free. This assumption is discussed in Sec.~\ref{sec:system-discussion}. 

{\bf Ideal vs. Noisy Perception.} Following \cite{han2018complexity}, for ideal perception, object matching is deterministic, and each current object correctly matches one of the goal objects. In this way, $\mathcal{M}_c^i(j)=\delta(j-j_i)$ if $o_c^i$ matches a goal object $o_g^{j_i}$, and otherwise $\mathcal{M}_c^i(j)=0$ for all $j$. Here $\delta$ denotes the Dirac delta function. Then the policy is to decide the object moving order. In contrast, with noisy perception, object matching is uncertain and can be wrong. This may be caused by uncertain detection, textured backgrounds, varied camera views, and etc. In tasks with 6-DoF pose transformations, the issue may be more severe than in 2-DoF or 3-DoF ones, especially when facing unseen objects in clutter. 
In this paper, we mainly tackle the rearrangement task with noisy perception.
} 

{
\subsection{Preliminaries}

}

\begin{figure}[t]
  \centering
  \includegraphics[width=0.8\linewidth]{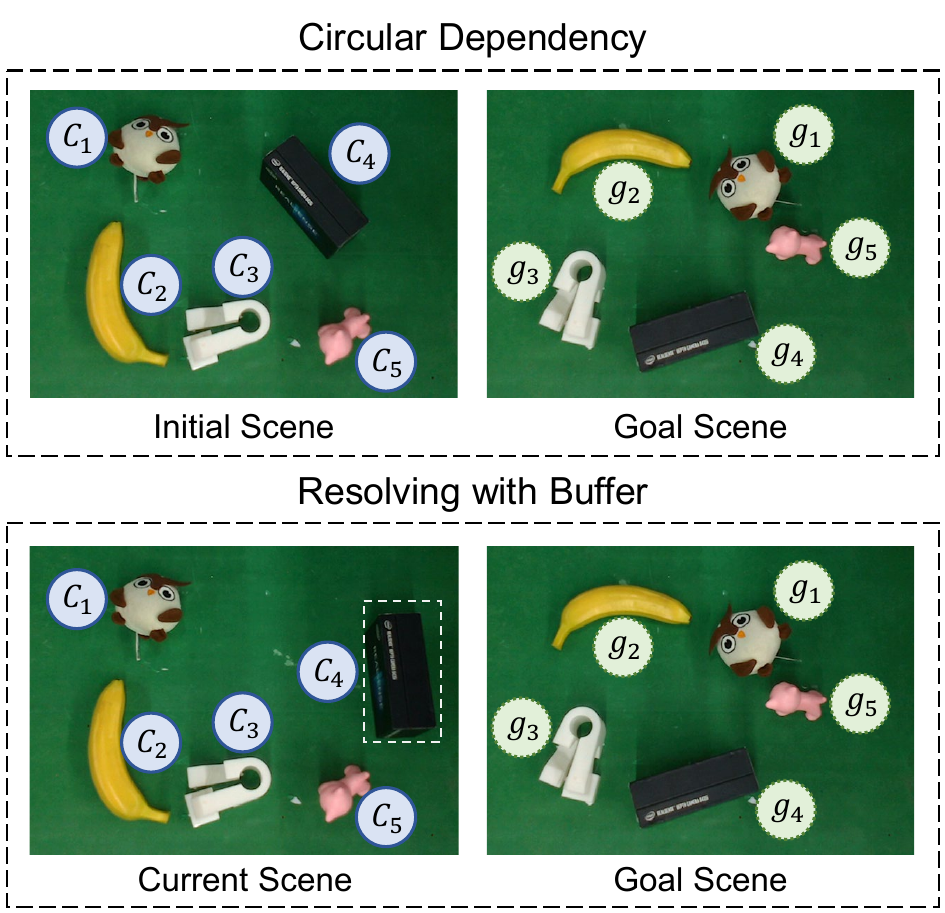}
  \vspace{-0.3cm}
  \caption{{An example to illustrate circular dependency and buffer. After moving $c_4$ to the buffer~(marked with the white box), the circular dependency breaks.}}
  \label{fig:dependency}
  \vspace{-0.6cm}
\end{figure}

{\bf Dependency.} For an object $o_c^{i_0}$, given $\mathcal{M}_c^{i_0}(j)$, its goal may be occupied by another object $o_c^{i_1}$, which we call $o_c^{i_0}$ depends on $o_c^{i_1}$. In this case, $o_c^{i_1}$ must be moved first before $o_c^{i_0}$ can be placed to its goal. Note that {one object can depend on multiple objects, but} $i_0\!\neq\!i_1$ {\it i.e.} if the goal of an object is only occupied by itself, there is no dependency on itself. The dependency relationship among current objects is represented as a dependency digraph $G_{dep}(V_{dep}, A_{dep})$, where $V_{dep}=\mathcal{O}_c$, and $(o_c^{i_0}, o_c^{i_1})\in A_{dep}$ indicates $o_c^{i_0}$ depends on $o_c^{i_1}$~\cite{gao2021minimizing}. Then each current object belongs to one of the following two sets:
\begin{equation}
\begin{aligned}
\mathcal{O}_{c,1}:\left\{o_c^{i_0}\big|\forall o_c^{i_1}\in\mathcal{O}_c, (o_c^{i_0}, o_c^{i_1})\notin A_{dep}\right\} \\
\mathcal{O}_{c,2}:\left\{o_c^{i_0}\big|\exists o_c^{i_1}\in\mathcal{O}_c, (o_c^{i_0}, o_c^{i_0})\in A_{dep}\right\}
\end{aligned}\nonumber
\end{equation}
where $i_0 \neq i_1$. Objects in {$\mathcal{O}_{c,1}$} can be directly moved to their goals. 
Instead, goals of objects in {$\mathcal{O}_{c,2}$} are initially occupied by others.
If an object is a non-goal object, it can always be moved outside, thus belonging to {$\mathcal{O}_{c,1}$}. Notably, an object may change its set membership across action steps, {\it e.g.} after an object $o_c^i \in {\mathcal{O}_{c,1}}$ is moved to the goal, another object depends on $o_c^i$ can change the membership from {$\mathcal{O}_{c,2}$} to {$\mathcal{O}_{c,1}$}. {For some objects in {$\mathcal{O}_{c,2}$}, it is possible to have circular dependencies~\cite{gao2021minimizing,gao2022fast}, {\it e.g.} among object $1, 2, 3, 4$ in Fig.~\ref{fig:dependency}. Once existing circular dependencies, some objects must be temporarily moved to the intermediate places to break the dependencies. These intermediate places are defined as buffers. A buffer can be a place that is currently not occupied and does not overlap with all goal regions. Fig.~\ref{fig:dependency} shows an illustration of circular dependency and buffer. }

{
{\bf Policy Structure.} We define the grasp action as $a_\mathcal{G}\!=\!i, i\!\in\!\{1,...,M\}$, which represents the grasps of $o_c^i$, and the place action as $a_\mathcal{P}\!=\!j^\prime,j^\prime\!\in\!\{1,...,N,N_o,N_b\}$, which represents the target of places: 
\begin{equation}
j^\prime = \begin{cases}
            j, & \text {place to $o_g^j$} \\ 
            N_o, & \text{place outside} \\
            N_b, & \text{place to the buffer} 
            \end{cases}
\end{equation}

To formulate the rearrangement policy, we represent the crop of the in-hand object after grasping as $o_h$, then the rearrangement policy is $\pi\left(a_{\mathcal{G}}, a_{\mathcal{P}}|\mathcal{O}_c, \mathcal{O}_g, o_h\right)$. $\pi$ can be factorized as grasp and place, which are denoted as $\pi_{\mathcal{G}}$ and $\pi_{\mathcal{P}}$:
\begin{equation}
\label{eq-gp}
\!\!\!\!\pi\!\left(a_{\mathcal{G}},\!a_{\mathcal{P}}|\mathcal{O}_c,\!\mathcal{O}_g,\!o_h\right)\!\!=\!\!\pi_{\mathcal{G}}\!\left(a_{\mathcal{G}}|\mathcal{O}_c,\!\mathcal{O}_g\right)\!\pi_{\mathcal{P}}\!\left(a_{\mathcal{P}}|\mathcal{O}_c,\!\mathcal{O}_g,\!o_h,\!a_{\mathcal{G}}\right) \!\!\!
\end{equation}
The grasp policy $\pi_{\mathcal{G}}$ is not conditioned on $o_h$ as it is taken after the object is grasped. The cue for $\pi_{\mathcal{G}}$ is the $M$-to-$N$ object matching $\mathcal{M}_c^i(j)$. Note that the place action can be derived from the in-hand object crop $o_h$ and the goal object crops $\mathcal{O}_g$. Thus, the cue for $\pi_{\mathcal{P}}$ is the 1-to-$N$ object matching $\mathcal{M}_h(j)$:
\begin{equation}
\label{eqn-place}
\!\!\!\pi_{\mathcal{P}}\!\left(a_{\mathcal{P}}|\mathcal{O}_c,\!\mathcal{O}_g,\!o_h,\!a_{\mathcal{G}}\right)\!=\!\pi_{\mathcal{P}}\!\left(a_{\mathcal{P}}|o_h,\!\mathcal{O}_g\right)
\end{equation}

\subsection{Policy with Ideal Perception}
\label{sec:opt-ideal}
With ideal perception, the in-hand object is certainly the object grasped by $a_{\mathcal{G}}$. Thus, the in-hand object matching $\mathcal{M}_h(j)$ is exactly $\mathcal{M}_c^{a_{\mathcal{G}}}(j)$ employed in the grasp policy:
\begin{equation}
\label{eqn-ideal_place}
\pi_{\mathcal{P}}\!\left(a_{\mathcal{P}}|o_h,\!\mathcal{O}_g\right)\!=\!\pi_{\mathcal{P}}\!\left(a_{\mathcal{P}}|a_{\mathcal{G}},\!\mathcal{O}_g\right)
\end{equation}}

{\bf Optimal Policy.} We represent the grasp policy to resolve objects with circular dependency in \cite{han2018complexity} as $\pi_{\mathcal{G}}^{fvs}$, and formulate a policy as: 
\begin{equation}
\label{eqn-pi0}
\pi^0=\pi_{\mathcal{G}}^0\bar{\pi}_{\mathcal{P}}^0
\end{equation}
where $\pi_{\mathcal{G}}^0$ and {$\bar{\pi}_{\mathcal{P}}^0$} are defined as: 
\begin{equation}
\pi_{\mathcal{G}}^0=\begin{cases} \sum_i\delta(a_{\mathcal{G}}=i|o_c^i \in \mathcal{O}_{c,1}), & \exists o_c^i,o_c^i\in\mathcal{O}_{c,1} \\
            \pi_{\mathcal{G}}^{fvs}, & \forall o_c^i,o_c^i\in\mathcal{O}_{c,2} \end{cases}
\end{equation}
\begin{equation}
\label{eqn-pi0-place}
\!\!\!\bar{\pi}_{\mathcal{P}}^0\!= 
                    \begin{cases} \delta(a_{\mathcal{P}}\!=\!j|a_{\mathcal{G}}), & \!o_c^{{a_{\mathcal{G}}}}\!\in\!\mathcal{O}_{c,1},\operatorname{max}_j \mathcal{M}_c^{a_{\mathcal{G}}}(j)\!\geq\!\zeta_g \\
                    \delta(a_{\mathcal{P}}\!=\!N_o|a_{\mathcal{G}}), & \!o_c^{{a_{\mathcal{G}}}}\!\in\!\mathcal{O}_{c,1},\operatorname{max}_j \mathcal{M}_c^{a_{\mathcal{G}}}(j)\textless\zeta_g \\
                    \delta(a_{\mathcal{P}}\!=\!N_b|a_{\mathcal{G}}), & \!o_c^{{a_{\mathcal{G}}}}\!\in\!\mathcal{O}_{c,2},\operatorname{max}_j \mathcal{M}_c^{a_{\mathcal{G}}}(j)\!\geq\!\zeta_g \end{cases}
\end{equation}
If there exist objects in {$\mathcal{O}_{c,1}$} in the current scene, $\pi^0$ iteratively moves them to their goals. {Otherwise, $\pi^0$ resolves objects with circular dependency as $\pi_{\mathcal{G}}^{fvs}$ by placing some of the objects with circular dependency into the buffers. Here $\bar{\pi}_{\mathcal{P}}^0$ is rule-based}, and the grasp and place policies are formulated with $\delta$ function for convenient comparisons of those with perception noise. Then we have:
\begin{theorem}
Given a tabletop object rearrangement problem from the configuration of $M$ objects $\mathcal{O}_c$ to that of $N$ objects $\mathcal{O}_g$, $\pi^0$ is an optimal policy {under ideal perception}.
\end{theorem} 
{
\begin{proof}
    See Appendix~A~\cite{appendix}.
\end{proof}
}

{\bf Non-uniqueness of Grasp.} By taking the optimal policy $\pi^0$, the order of grasping objects $o_c^i \in {\mathcal{O}_{c,1}}$ into their goals does not affect the total pick-n-place steps. Thus, the optimal grasp policy is multimodal. Particularly, the optimal distribution of grasping objects in {$\mathcal{O}_{c,1}$} is uniform. Besides, the optimal rearrangement policy $\pi^0$ may not be unique as well.

{\subsection{Policy with Noisy Perception}
\label{sec:opt-noise}
}
The most affected component of the policy $\pi^0$ in Eqn.~\ref{eqn-pi0} is $\mathcal{M}_c^i(j)$. {Due to perception noise, $\mathcal{M}_c^i(j)$ is uncertain, possibly resulting in incorrect matching for the grasp policy $\pi_\mathcal{G}^0$.} For the place policy {$\bar{\pi}_\mathcal{P}^0$}, the 1-to-$N$ object matching $\mathcal{M}_h(j)$ derived from $\mathcal{M}_c^{a_\mathcal{G}}(j)$ may also be wrong. These wrong object matching can cause additional steps, impacting the optimality. 

{\bf Decoupled Perception.} Recalling Eqn.~\ref{eqn-ideal_place}, we find that the 1-to-$N$ object matching $\mathcal{M}_h(j)$ may not be necessary to derive from the $M$-to-$N$ object matching as $\mathcal{M}_c^{a_{\mathcal{G}}}(j)$:
\begin{equation}
\label{eqn-noise_place}
\pi_{\mathcal{P}}\!\left(a_{\mathcal{P}}|o_h,\!\mathcal{O}_g\right)\!\neq\!\pi_{\mathcal{P}}\!\left(a_{\mathcal{P}}|a_{\mathcal{G}},\!\mathcal{O}_g\right)
\end{equation}
In this way, the pathway of perception noise impacting the two policies is decoupled. Thus, it is possible to improve the {in-hand object matching} independently after grasping an object.

{
{\bf Improvement by In-hand Object Matching.} We formulate a policy:
\begin{equation}
\label{eqn-pi1}
\pi^1=\pi_{\mathcal{G}}^0\bar{\pi}_{\mathcal{P}}^1
\end{equation}
where $\bar{\pi}_{\mathcal{P}}^1$ is the place policy conditioned on $\mathcal{M}_h(j)$:
\begin{equation}
\label{eqn-pi1-place}
\!\!\!\bar{\pi}_{\mathcal{P}}^1\!=\! 
                    \begin{cases} \delta(a_{\mathcal{P}}\!=\!j|o_h), & \!o_h\!\in\!\mathcal{O}_{c,1},\operatorname{max}_j \mathcal{M}_h(j)\!\geq\!\zeta_g \\
                    \delta(a_{\mathcal{P}}\!=\!N_o|o_h), & \!o_h\!\in\!\mathcal{O}_{c,1},\operatorname{max}_j \mathcal{M}_h(j)\textless\zeta_g \\
                    \delta(a_{\mathcal{P}}\!=\!N_b|o_h), & \!o_h\!\in\!\mathcal{O}_{c,2},\operatorname{max}_j \mathcal{M}_h(j)\!\geq\!\zeta_g \end{cases}
\end{equation}
Similar to $\bar{\pi}_{\mathcal{P}}^0$, $\bar{\pi}_{\mathcal{P}}^1$ is also rule-based. 
We show that improvement of the in-hand object matching $\mathcal{M}_h(j)$ is valuable for the improvement of task-level performance:
\begin{theorem}
Given a tabletop object rearrangement problem from the configuration of $M$ objects $\mathcal{O}_c$ to that of $N$ objects $\mathcal{O}_g$, let total pick-n-place steps of $\pi^0$~(Eqn.~\ref{eqn-pi0}) and $\pi^1$~(Eqn.~\ref{eqn-pi1}) be $\mathcal{F}^{\pi^0}(\mathcal{O}_c|\mathcal{O}_g)$ and $\mathcal{F}^{\pi^1}(\mathcal{O}_c|\mathcal{O}_g)$, if $\mathcal{M}_h(j)$ has less matching error than $\mathcal{M}_c^{a_\mathcal{G}}(j)$ under noisy perception, then $\mathcal{F}^{\pi^0}\geq\mathcal{F}^{\pi^1}$.
\end{theorem}

\begin{proof}
    See Appendix~B~\cite{appendix}.
\end{proof}
}
{\bf Grasp Supervision.} Note that the optimal grasp with ideal perception is multimodal, now the noisy object matching and grasping capability make the optimal policy a more complex distribution, {making it hard to define rule-based solutions. Therefore, the optimal grasp policy with noisy perception should be learned.} For policy learning, if we cannot sample the grasp from this distribution as labels for supervision, there must be a bias injected into the resultant learned policy.

{In general, Theorem 1 reveals the optimal policy structure and the decoupled nature of grasp and place. Based on the structure prior, Theorem 2 indicates the importance of improving in-hand object matching for task-level performance.}

{
\subsection{Rearrangement Policy with Structure Prior}
\label{sec:system}
Guided by the insights, we propose to improve the in-hand object matching. Inspired by human behavior, we find that active perception~\cite{bajcsy1988active,aloimonos1988active,chen2011active} is an effective technique to improve perception. Specifically, active perception is to sense in a closed loop, where the robot takes actions based on the perception, and to refine the in-hand object matching. As for the grasp policy, we employ reinforcement learning guided by task-level rewards, to avoid the drawback of sole supervision.

{\bf Active Perception of In-hand Object Matching.} We introduce an active perception policy to actively reorient the in-hand object for better object matching. This policy is named as see policy $\pi_{\mathcal{S}}$, which predicts the action of active perception $a_{\mathcal{S}}$ only conditioned on the object matching between the in-hand object and the goal objects. Therefore, $\pi_{\mathcal{S}}$ can be learned independently from the whole rearrangement system, thus easing the reinforcement learning. 

{\bf See and Place.} By incorporating $\pi_{\mathcal{S}}$, $\pi_{\mathcal{P}}$ in Eqn.~\ref{eqn-noise_place} is factorized as a learning-based $\pi_{\mathcal{S}}$ and a rule-based $\bar{\pi}_{\mathcal{P}}$:
\begin{equation}
\label{eq-gsp}
\pi_{\mathcal{P}}\!\left(a_{\mathcal{S}},a_{\mathcal{P}}|o_h,\! \mathcal{O}_g\right)\!=\!\pi_{\mathcal{S}}\!\left(a_{\mathcal{S}}|o_h,\! \mathcal{O}_g\right)\!\bar{\pi}_{\mathcal{P}}\!\left(a_{\mathcal{P}}|o_h,\! \mathcal{O}_g,\!a_{\mathcal{S}}\right)
\end{equation}
Upon getting a confident object matching, the place policy $\bar{\pi}_{\mathcal{P}}$ follows Eqn.~\ref{eqn-pi1-place} to place the in-hand object to the goal. To achieve the 6-DoF aligned placement, the place policy employs the relative pose estimation between the in-hand and goal object patches after the match is determined.

{\bf Grasp.} Based on the see and place policies, we further build the learning-based grasp policy $\pi_{\mathcal{G}}$. Given RGB-D images of the current and the goal scenes, object image crops are obtained by an open-set detection module~\cite{zhou2023open}. Then the grasp policy jointly takes object crops of current and goal scenes, as well as candidate grasps as input, to output a selected grasp to execute. In this way, the grasp policy is aware of the noisy object matching, placement, and grasp capability. Besides, we employ reinforcement learning guided by task-level rewards, to avoid the drawback of supervision.  

{\bf Overall Behavior.} As a whole, our system, namely GSP, decomposes the object rearrangement process into two loops: an inner loop of {see} for {improving the in-hand object matching} and an outer loop for the planning of grasp and place~(Fig.~\ref{fig:overview}). This dual-loop structure relieves the sparse reward for RL in long-horizon tasks, and focuses on task-level performance. At inference time, GSP predicts an action sequence of grasp, see and place to manipulate an object for each step. The system imitates a human-like household behavior of rearranging a clutter of objects: According to a desired goal configuration, it first plans a grasp to pick up an object. Once the object is in hand, it confirms object matching with the desired configuration by reorienting the object to examine it from different views. Then, the object is rearranged to a planned place pose. If the in-hand object is wrongly matched, or needs to be placed in a buffer due to the occupancy of the goal, it will be considered in the subsequent process until it is placed in the goal region. Finally, the rearrangement process ends if the current scene is consistent with the goal scene.
}

{
\vspace{-0.2cm}
\subsection{Discussions}
\label{sec:system-discussion}
{{\bf Comparative Analysis.}} Based on the above analysis, we discuss the three most relevant systems dealing with unknown object rearrangement: NeRP~\cite{qureshi2021nerp}, IFOR~\cite{goyal2022ifor} and Selective Rearrangement~\cite{tang2023selective}. NeRP learns grasp and place from demonstrations collected by a model-based expert. IFOR deploys a rule-based planner that greedily picks the objects with non-occupied goals and large pose transformations. Selective Rearrangement selects objects by comparing graph editing distance, which also follows a greedy strategy. These strategies act as per step labels, and it is hard to verify whether they follow the correct optimal grasp distribution, probably bringing bias to the learned policy. Lastly, none of them mention the importance of the decoupled perception for task-level improvement, thus the performance is bounded by the initial difficult $M$-to-$N$ object matching. 

{\bf 6-DoF Transformation Tasks.} For 6-DoF transformation tasks, things can be more challenging than 2-DoF or 3-DoF ones, which is reflected in the object matching and placement. For object matching, the wrong cases can be more frequent. Thereby active perception is especially beneficial for 6-DoF transformation tasks. For object placement, objects may need to be reoriented several times to achieve their goal poses. However, firstly, it does not affect the grasp policy {\it i.e.} the object moving order. Secondly, for the place of such kind of objects, the additional steps of object reorientation can be the same for all policies. Therefore, the theorems still hold. That is, the optimality of $\pi^0$ and the decoupling property of noisy perception remain the same. Given the second point, we assume that each object can be moved to its goal pose with one step of pick-n-place when the goal is free. 
}

\begin{figure}[t]
  \centering
  \includegraphics[width=\linewidth]{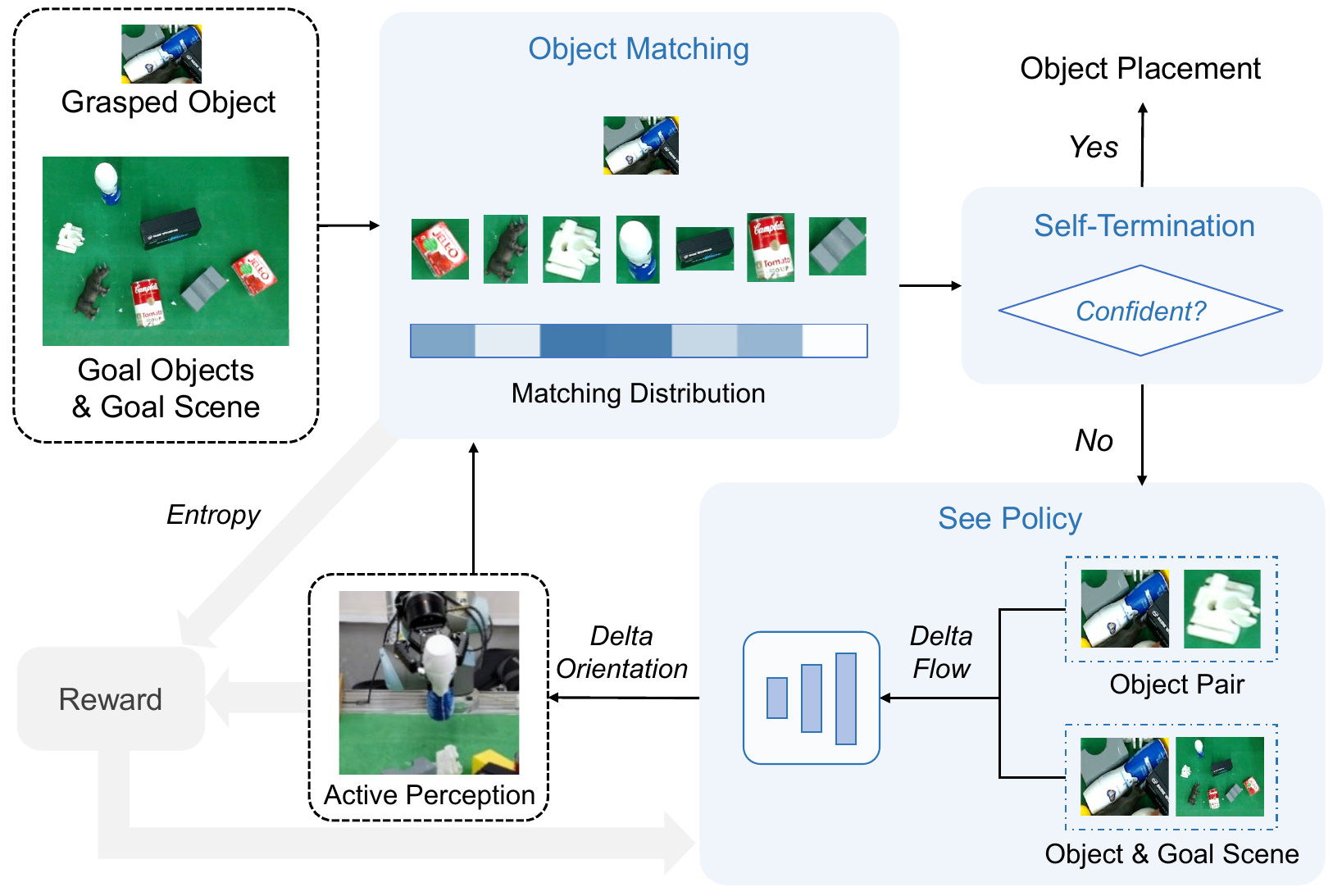}
  \vspace{-0.7cm}
  \caption{{Architecture of the inner loop, which consists of four key components: object matching, self-termination, see policy, and reward formulation. The four components form a closed loop for self-confident object matching.}}
  \label{fig:as-architecture}
  \vspace{-0.6cm}
\end{figure}

{
\vspace{-0.2cm}
\section{Inner Loop: See}
\label{active_seeing}
}

We propose to learn {see} policy $\pi_{\mathcal{S}}$ as the inner loop to improve {the in-hand object matching}. By actively reorienting the in-hand object, $\pi_{\mathcal{S}}$ gets multi-view observations for self-confident object matching. Different from the previous works that implement ResNet~\cite{he2016deep} for object matching, we employ the foundation model CLIP~\cite{radford2021learning} to capture object-level features and enable zero-shot generalization. Additionally, CLIP is utilized to densify the reward and self-termination. To guide the fine-grained reorientation, we follow IFOR~\cite{goyal2022ifor} to use optical flow for policy observation, which extracts fine-grained correspondences of images. 

Fig.~\ref{fig:as-architecture} illustrates the architecture of the inner loop, which consists of four key components: object matching, self-termination, {see} policy, and reward formulation. Given the in-hand object crop {$o_h$}, the object matching module pairs it with the goal objects $\mathcal{O}_g$, followed by the self-termination module determining whether the matching result is confident. If not confident, the {see} policy predicts the delta orientation of the end-effector to rotate the grasped object to another view for matching. To learn the {see} policy, a dense reward mechanism based on the matching results is proposed. Our system will jump out of the inner loop to object placement as soon as it receives a confident matching. 
\vspace{-0.2cm}
\subsection{Object Matching and Self-confidence}
\label{matcher}
Once successfully grasping an object {$o_h$}, the robot is supposed to identify whether the grasped object is in the goal scene and, if so, identify the matched goal object in $\mathcal{O}_g$. 

{\bf Modeling of $\mathcal{M}_{h}(j)$.} Considering $\Phi$ as the image encoder, the correlations between the grasped object {$o_h$} and the goal objects $\mathcal{O}_g$ can be represented as the cosine similarities of their image embeddings. We formulate the matching from {$o_h$} to $\mathcal{O}_g$ as a one-shot classification problem, then the matching distribution $\mathcal{M}_{h}(j)$ can be represented as the softmax distribution of correlations between {$o_h$} and $\mathcal{O}_g$:
\begin{equation}
\label{eqn-visual_emb}
e_o=\frac{\Phi(o)}{|\Phi(o)|}, o \in \{o_h, \mathcal{O}_g\}\\
\end{equation}
\begin{equation}
\label{eqn-match_dist}
\mathcal{M}_{h}(j)=\mathrm{Softmax}\left\{\left\langle e_{o_h}, e_{o_g^j}\right\rangle\right\}_{j=1, ., N}
\end{equation}
where $e_o$ denotes the normalized image embedding of the object image crop $o$, $\langle\cdot, \cdot\rangle$ represents the operation of cosine similarity {\it i.e.} inner-product.

{\bf Matching Confidence.} Given the matching distribution $\mathcal{M}_{h}(j)$, the matched goal object is formulated as:
\begin{equation}
o_g^{j_h}=o_g^{\operatorname{argmax}_j \mathcal{M}_{h}(j)}
\end{equation}
Accordingly, the matching score and distribution entropy are formulated as:
\begin{equation}
\label{eqn-score}
m_{h}=\operatorname{max}_j \mathcal{M}_{h}(j)
\end{equation}
\begin{equation}
\label{eqn-ent}
H_h=-\sum_j \mathcal{M}_{h}(j) \log \mathcal{M}_{h}(j)
\end{equation}
Note that $m_{h}$ manifests the matching confidence, and $H_h$ can represent the matching uncertainty.

{\bf CLIP Feature based Matching.} We propose to leverage the pre-trained vision-language model CLIP~\cite{radford2021learning} for object matching, which aligns text and image features into a common space on vastly large-scale datasets. With semantic priors, CLIP cares more about the global information of the object image crop, and endows the robot with the capability of zero-shot generalization to unseen objects. 

{\bf CLIP-Adapter.} In order to achieve better object matching performance in our environments, we further finetune CLIP with CLIP-Adapter~\cite{gao2021clip} which adopts a simple residual transformation layer over the CLIP feature tokens. Specifically, we apply a visual adapter $\mathrm{A}_v(\cdot)$ with the gating ratio $\epsilon=0.2$ to balance and mix the knowledge from the original tokens and the adapter outputs. Eqn.~\ref{eqn-match_dist} is accordingly updated:
\begin{equation}
\label{eqn-visual_emb_finetune}
{\hat e}_o=\epsilon \mathrm{A}_v(e_o)+(1-\epsilon) e_o
\end{equation}
\begin{equation}
\label{eqn-match_dist_new}
\mathcal{M}_{h}(j)=\mathrm{Softmax}\left\{\left\langle {\hat e}_{o_h}, {\hat e}_{o_g^j}\right\rangle\right\}_{j=1, ., N}
\end{equation}
\vspace{-0.5cm}

\begin{figure}[t]
  \centering
  \includegraphics[width=\linewidth]{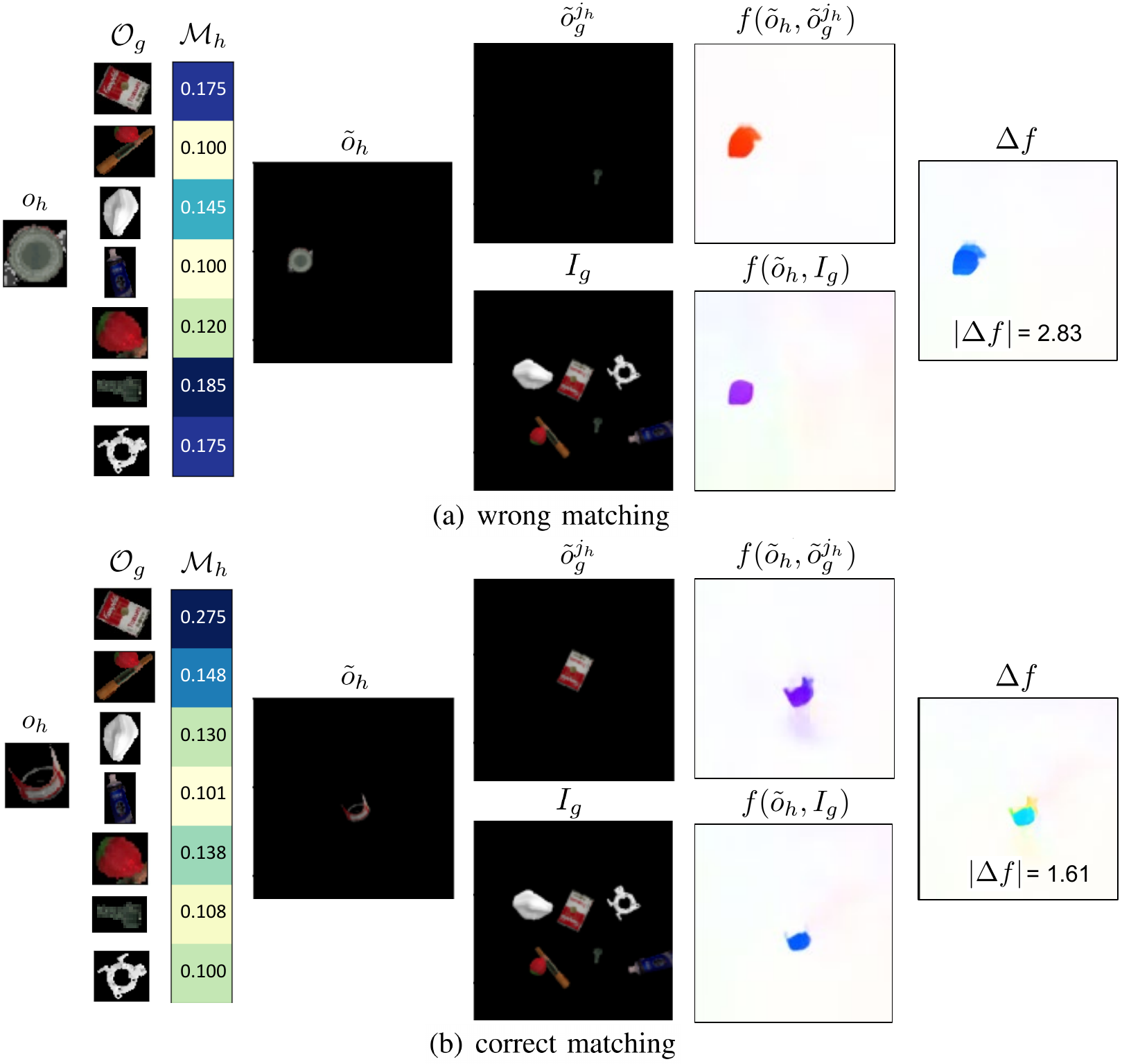}
  \vspace{-0.7cm}
  \caption{{State representation of the {see} policy. The robot is supposed to rotate the tomato soup can for a confident matching to a set of goal objects. We visualize two cases of scene representation including (a) wrong matching and (b) correct matching with the same flow field color coding in \cite{teed2020raft}. For each case, optical flows $f$ are generated between the grasped object padding image $\tilde{o}_h$ and the current-matched object padding image $\tilde{o}_g^{j_h}$~(object flow), as well as the goal image $I_g$~(global flow). Then their delta flow $\Delta f$ is the state representation of the see policy, marked with the average magnitude. Note that the image crop of the grasped object is disturbed by the gripper, thus bringing noise for matching.}}
  \label{fig:state-representation}
  \vspace{-0.7cm}
\end{figure}

{
\subsection{See Policy}
\label{sec:as_policy}
}
Based on the object matching result $\mathcal{M}_h(j)$, the {see} policy $\pi_\mathcal{S}$ is to rotate the grasped object {$o_h$} for more reliable matching~({\it i.e.} to improve the matching confidence), and determine when to stop the execution of {see}. There are several challenges to learning such a policy, including what to formulate the policy state, how to guide policy training, and when to stop the policy execution. 

{{\bf Policy State.}
For the active perception process, the policy state is supposed to represent the differences of the in-hand object in different views. Although we leverage CLIP~\cite{radford2021learning} for object matching, it is hard to provide such representation, as CLIP focuses more on high-level semantic information ({\it i.e.} object-level information). }Following \cite{goyal2022ifor}, we generate dense pixel-wise correspondences between images using RAFT~\cite{teed2020raft}. Such representation contains fine-grained matching information that is sensitive to pose transformation. These pixel-wise correspondences guide the prediction of fine-grained actions and manifest the object-level matching information. 

To generate the optical flow of the grasped object, object representation is processed into the padding image, which is of the same size as the RGB image of the scene, but masked outside the bounding box of the grasped object. Denoting the padding image of the grasped object as {$\tilde{o}_h$}, the padding image of matched goal object as $\tilde{o}_g^{j_h}$, we represent the object optical flow from {$\tilde{o}_h$} to $\tilde{o}_g^{j_h}$ as {$f(\tilde{o}_h, \tilde{o}_g^{j_h})$}, and the global optical flow from {$\tilde{o}_h$} to $I_g$ as {$f(\tilde{o}_h, I_g)$}. We notice that these two optical flows are similar if the matching is of high confidence. That is, the delta flow of {$f(\tilde{o}_h, \tilde{o}_g^{j_h})$} and {$f(\tilde{o}_h, I_g)$} indicates the extent to which dense correspondences are affected by non-matched objects. Thus, the delta optical flow can serve as the representation of pixel-wise matching uncertainty between {$o_h$} and $\mathcal{O}_g$, which is formulated as the state of the {see} policy $\pi_\mathcal{S}$:
\begin{equation}
\Delta f=f(\tilde{o}_h, \tilde{o}_g^{j_h})-f(\tilde{o}_h, I_g)
\end{equation}

Fig.~\ref{fig:state-representation} shows two cases of state representation for wrong matching and correct matching. The direction and magnitude of optical flow are visualized with color coding with same as \cite{teed2020raft}, and the darker color mirrors the larger magnitude. In these two cases, the robot is supposed to rotate the tomato soup can for a confident matching to a set of goal objects. We can see that when the matching is correct, the average magnitude of the delta flow is smaller than that of the wrong matching, which reflects the influence of non-matched objects. Note that the image crop of the grasped object is disturbed by the gripper, thus bringing noise for matching. 

{{\bf Policy Action.}} Given the state $\Delta f$, the policy $\pi_\mathcal{S}$ predicts the active {perception} action $a_{\mathcal{S}}$, which is defined as the continuous 3-DoF delta orientation angles~({\it i.e.} pitch, row, yaw) of the end effector from its current pose. The execution of $a_{\mathcal{S}}$ leads to reorientation of the in-hand object {$o_h$}, creating a new view for object matching. As it is hard to annotate an in-hand reorientation action with its {``see''} quality, we learn $\pi_\mathcal{S}$ through reinforcement learning. 

{\bf Reward Formulation.} Considering that the target of the {see} policy is to improve the matching confidence, we propose to guide the policy training with the delta entropy of matching distribution, which can represent the change of matching uncertainty. If the matching distribution entropy is reduced after rotating the grasped object, the rotating action is regarded as effective for more confident matching, providing dense rewards to improve the sample efficiency. We define the $\mathds{1}_{m}$ for the matching result (Eqn.~\ref{eqn-match_success}), and an episode ends if the object is correctly matched. Also, $\mathds{1}_{r}$ denotes the rotate completion result (Eqn.~\ref{eqn-rotate_success}). If the robot successfully plans and executes the predicted action, then the rotation process is regarded as completed. 
\begin{equation}
\label{eqn-match_success}
\mathds{1}_{m}= \begin{cases}1, & \text{if correctly matched} \\ 0, & \text{otherwise}\end{cases}
\end{equation}
\begin{equation}
\label{eqn-rotate_success}
\mathds{1}_{r}= \begin{cases}1, & \text {if completely rotated} \\ 0, & \text{otherwise}\end{cases}
\end{equation}

Our reward function is defined as follows.
\begin{equation}
\label{eqn-reward}
R_{\mathcal{S}}=\mathds{1}_{m}+\Delta H_h-P_{\mathcal{S}} \\
\end{equation}
where $\Delta H_h$ represents the delta entropy of the matching distribution before and after active {perception} (Eqn. \ref{eqn-deltaent}). And $P_{\mathcal{S}}$ is the penalty function defined as the weighted combination of rotation failure and action magnitude (Eqn.~\ref{eqn-punish}).
For the {see} policy, large object rotation ({\it i.e.} large action magnitude) may result in unsafe execution {\it e.g.} self-collision. Therefore, a weighted penalty of the action magnitude $\left\|a_{{\mathcal{S}}}\right\|$ is conducted.
\begin{equation}
\label{eqn-deltaent}
\Delta H_h=H_h^{\text{before}\_\text{see}}-H_h^{\text{after}\_\text{see}} \\
\end{equation}
\vspace{-0.5cm}
\begin{equation}
\label{eqn-punish}
P_{\mathcal{S}}=\lambda \left(1-\mathds{1}_{r}\right)+\mu\left\|a_{{\mathcal{S}}}\right\| \\
\end{equation}

{\bf Self-Termination.} During policy training, the robot receives the ground-truth terminal signal of policy execution {\it i.e.} the robot gets ground-truth feedback of matching correctness. However, in actual application, there is no ground-truth terminal signal and the policy is supposed to self-terminate the execution. For active {perception}, the execution termination depends on the matching confidence of the grasped object to the goal objects {\it i.e.} $m_{h}$. If the matching confidence exceeds a threshold $\zeta_m$, the matching is regarded as valid, and then the object will be placed. This threshold $\zeta_m$ is namely the confidence-based terminal signal. It can be derived from the matching distribution, or be learned by the training scheme.

{
\subsection{Implementation Details of CLIP Finetuning}
\label{as_details}
To achieve better matching performance, we finetune CLIP with CLIP-Adapter~\cite{gao2021clip}. 

\label{sec:as-data-collection}
{\bf Data Collection.} We collect 3k samples for CLIP finetuning in PyBullet~\cite{coumans2021} with a statically-mounted camera of Intel RealSense L515 overlooking the tabletop. Camera images are projected as heightmaps as in \cite{zeng2021transporter}. The object models are from GraspNet-1Billion~\cite{fang2020graspnet}. During the data collection process, 5$\sim$7 objects are sampled to be randomly dropped into the workspace to generate object crops as the support set. Each of them, as well as 5 other objects are separately dropped into the workspace with random poses to generate object crops as queries. Each query object crop pairs with the support set to form a sample. For simplicity, we generate object bounding boxes for object image crops from the mask image in Pybullet, of which each pixel denotes the index of the object visualized in the camera. We label the samples that have a matched object in the support set with one-hot encodings, while labeling others with uniform distributions. 

{\bf Training Details.} The adapter is trained with Kullback-Leibler divergence and Adam optimizer for 400 epochs. Step-scheduled learning rates are conducted from the initial learning rate at $1\times 10^{-3}$, with the decay ratio as 0.5 and the step size of 50 epochs.

\subsection{Implementation Details of Policy Learning}
We learn the {see} policy $\pi_\mathcal{S}$ with reinforcement learning in simulation. Our training environment involves a UR5 arm, a ROBOTIQ-85 gripper and a statically-mounted camera of Intel RealSense L515 in PyBullet. 

{\bf Data Collection.} During the training process, we randomly sample an object to grasp for active perception and 5$\sim$7 objects containing the object to grasp as the goal object set. For each episode, the object to grasp is dropped into the workspace with a random pose, and graspnet~\cite{fang2020graspnet} predicts a set of feasible 6-DoF grasp poses. Then the robot randomly chooses a grasp to execute until it picks up the object, and conducts object matching. If the matching from the finetuned CLIP is incorrect, the robot plans to a pre-defined pose as the initial pose for see, followed by several predicted see actions upon this initial pose. The episode is regarded as successful if the matching is correct within 5 see steps. 

{\bf Hyperparameters.} For the weighted penalty, we use $\lambda\!=\!0.2, \mu\!=\!0.18$. For confidence-based terminal signals, 
we use the threshold of $\zeta_m\!=\!\frac{1}{N}+0.12$ for finetuned CLIP. More details can be seen in Appendix~C~\cite{appendix}.
}

\section{Outer Loop: Grasp and Place}
Based on the analysis in Sec.~\ref{sec:opt-noise}, the perception of grasp and place can be decoupled, which enables the independent learning of grasp and see. Given the {see} policy, we build the outer loop of grasp and place for the whole rearrangement process. Since the learned {see} policy $\pi_{\mathcal{S}}$ provides closed-loop perception for place, the place policy {$\bar{\pi}_{\mathcal{P}}$} can follow Eqn.~\ref{eqn-pi1-place}. Instead, the grasp policy $\pi_{\mathcal{G}}$ needs to consider the uncertainty of object matching and grasping, which is hard to analytically formulate. Therefore, we focus on the learning of $\pi_{\mathcal{G}}$. 

\begin{figure}[t]
  \centering
  \includegraphics[width=\linewidth]{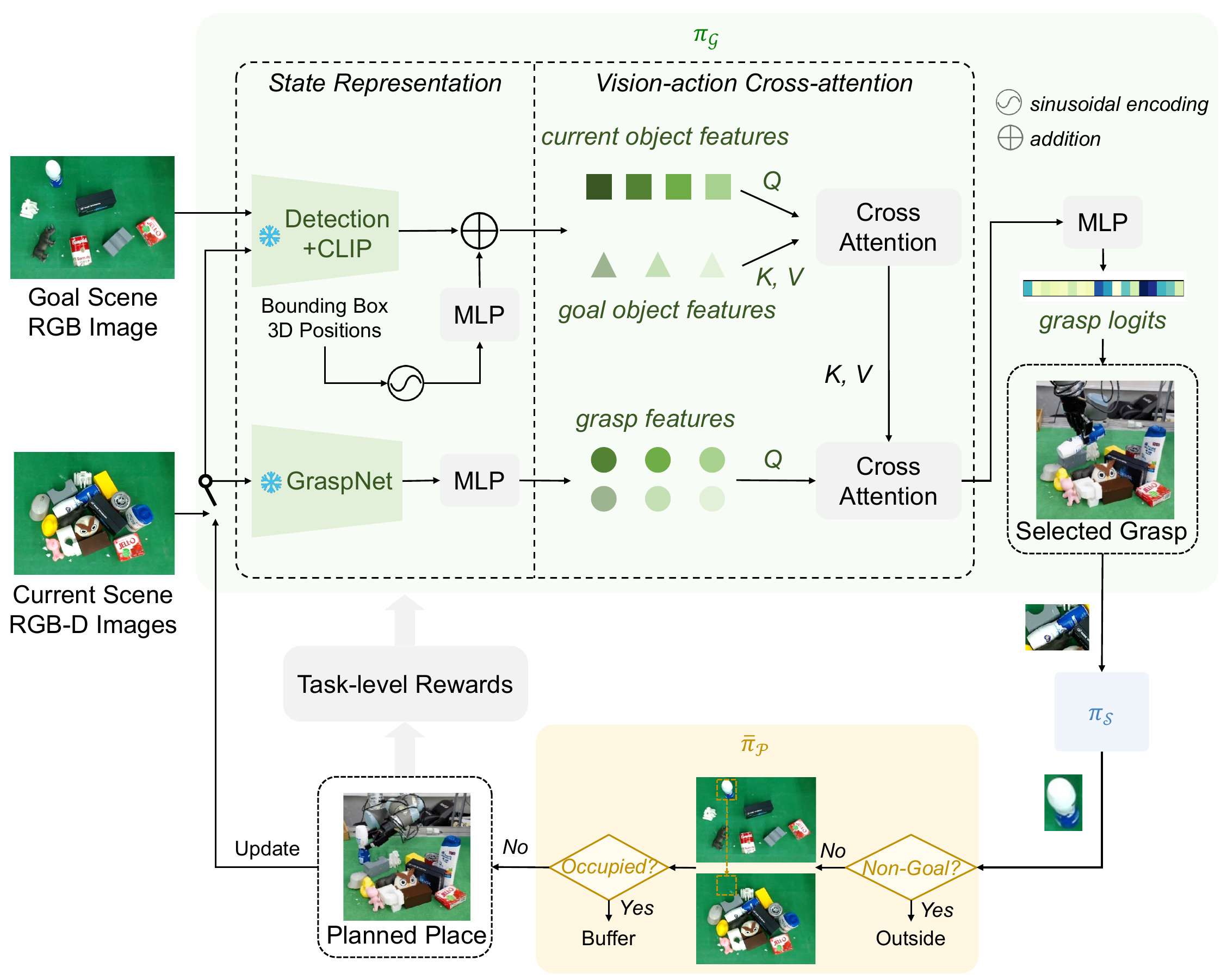}
  \vspace{-0.6cm}
  \caption{Architecture of the outer loop, which consists of grasp and place policy. The grasp policy $\pi_{\mathcal{G}}$ takes as input current and goal images, as well as candidate grasp poses to select a grasp to execute. Specifically, it jointly models object selection and action planning as available grasp evaluation. $\pi_{\mathcal{G}}$ fuses current object features, goal object features and grasp action features by cross-attention. The place policy {$\bar{\pi}_{\mathcal{P}}$} follows Eqn.~\ref{eqn-pi1-place} based on the matching result of the {see} policy $\pi_{\mathcal{S}}$.}
  \label{fig:outer-loop}
  \vspace{-0.6cm}
\end{figure}

As shown in Fig.~\ref{fig:outer-loop}, the grasp policy $\pi_{\mathcal{G}}$ takes as input the current and goal images as well as the candidate grasps, and selects a grasp to execute. Instead of separately conducting object selection and grasp prediction of the corresponding object, we directly conduct grasp selection of the currently available grasp actions. This is more feasible in clutter, avoiding the problem of selecting the object without an available grasp. After grasping an object, the {see} policy $\pi_{\mathcal{S}}$ obtains a self-confident matching, followed by the place policy {$\bar{\pi}_{\mathcal{P}}$} moving the object as Eqn.~\ref{eqn-pi1-place}. Then the current image is updated for the next grasp, forming the outer loop. The grasp policy is trained through RL guided by task-level rewards based on the place result, avoiding the drawback of per-step supervision. Notably, built on learned $\pi_{\mathcal{S}}$ and {$\bar{\pi}_{\mathcal{P}}$}, the sparse reward problem in the long-horizon task is relieved.

\subsection{Grasp Policy}
{\bf Policy Architecture.} Fig.~\ref{fig:outer-loop} shows the policy architecture of $\pi_{\mathcal{G}}$. The grasp policy takes as input the current object image crops $\mathcal{O}_c$ as well as the goal object image crops $\mathcal{O}_g$ from the detection module, and generates visual features of all image crops by finetuned CLIP image encoder. Each object visual feature is fused with the embedding of the bounding box center position to generate object visual-position feature. Also, we utilize the pre-trained graspnet~\cite{fang2020graspnet} to yield a set of grasp poses $\mathcal{A}_\mathcal{G}=\{p_{\mathcal{G}}^k\}_{k=1,...,L}$ that are encoded to action features. Given these feature embeddings, we fuse vision and action information with cross-attention. Then a followed MLP takes as input the cross-attention features, and outputs the grasp logits for grasp selection. 

{{\bf Policy State.}} We obtain object bounding boxes to generate object image crops as object-centric representations, and yield visual features through finetuned CLIP described in Sec.~\ref{matcher}. Specifically, given the RGB image, bounding boxes are extracted from the detection module, and object crops are obtained by cropping the raw RGB image with corresponding bounding boxes. Then each object crop is fed into the finetuned CLIP image encoder to get the object visual feature {${\hat e}_{o}$} by {Eqn.~\ref{eqn-visual_emb_finetune}}. Also, we obtain 3D position $pos_o=(x, y, z)$ of each bounding box center by transforming pixel coordinates to world coordinates with camera extrinsics for $(x, y)$ and getting the depth of the center pixel for $z$. The 3D position is then projected into a nonlinear space (positional embedding as in \cite{mildenhall2020nerf}), followed by an MLP to encode the object position embedding $e_{pos}$:
\begin{equation}
\label{eqn-pos_emb}
e_{pos}=\mathrm{MLP_1}(\mathrm{PE}(pos_o))
\end{equation}
For the current scene and the goal scene, object visual features and position embeddings are generated in the same way, denoted as $\{{\hat e}_{o_c^i}\}_{i=1,...,M}$, $\{{e}_{pos_c^i}\}_{i=1,...,M}$, $\{{\hat e}_{o_g^j}\}_{j=1,...,N}$, $\{{e}_{pos_g^j}\}_{j=1,...,N}$ respectively. Besides, we adopt the pre-trained model of graspnet~\cite{fang2020graspnet} to get a set of grasp poses of the whole scene. Graspnet~\cite{fang2020graspnet} is a 6-DoF grasp detector taking as input a scene point cloud, and predicting $L$ 6-DoF grasp poses $\{p_{\mathcal{G}}^k\}_{k=1,...,L}$, which are encoded into $L$ action embeddings $\{e_{p_{\mathcal{G}}^k}\}_{k=1,...,L}$ by MLP:
\begin{equation}
\label{eqn-grasp_emb}
e_{p_{\mathcal{G}}^k}=\mathrm{MLP_2}(p_{\mathcal{G}}^k)
\end{equation}

{{\bf Policy Action.}} We fuse the embeddings from vision and action by transformer’s attention mechanism \cite{vaswani2017attention}: 
\begin{equation}
\text{Attention}(Q, K, V)=\text{Softmax}\left(Q K^T\right) V \\
\end{equation}
where $Q, K, V$ denote query, key and value respectively. 

First, we fuse each object's visual feature with its position feature. Each object position embedding $e_{pos}$ is fused with the corresponding object visual feature {${\hat e}_{o}$} by adding: 
\begin{equation}
\label{eqn-op_fusion}
e_{op} = {\hat e}_{o} + e_{pos}
\end{equation}
Object visual-position features $e_{op}$ of the current scene and the goal scene are fused by a cross-attention module with $M$ current scene object visual-position features $\{e_{{op}_c^i}\}_{i=1,...,M}$ as queries, $N$ goal scene object visual-position features $\{e_{{op}_g^j}\}_{j=1,...,N}$ as keys and values, generating $M$ visual fusion features $\{e_{{op}_{cg}^i}\}_{i=1,...,M}$. These features implicitly represent the information of object matching and placement. 

Then, another cross-attention module takes $L$ grasp features $\{e_{p_{\mathcal{G}}^k}\}_{k=1,...,L}$ as queries and $M$ visual fusion features $\{e_{{op}_{cg}^i}\}_{i=1,...,M}$ as keys and values. In this way, we combine the information of object matching, placement as well as candidate grasps.

After two stages of feature fusion, we get $L$ vision-action cross-attention features $\{e_{{opa}^k}\}_{k=1,...,L}$, which jointly consider object matching, placement as well as grasp capability. These features are then fed into an MLP network to output logits of $L$ grasp poses $\{l_a^k\}_{k=1,...,L}$:
\begin{equation}
\label{eqn-logits}
l_a^k =\mathrm{MLP_3}(e_{{opa}^k})
\end{equation}

Finally, the grasp pose with the highest logit is chosen for execution {\it i.e. }$p_{\mathcal{G}}^{k^\star}$:
\begin{equation}
k^\star =\operatorname{argmax}_k l_a^k 
\end{equation}

{\bf Reward Formulation.} According to Sec.~\ref{sec:opt-noise}, the grasp distribution is multimodal conditioned on subjective uncertainty factors, which is hard to learn by supervision. Simply imitating the demonstration data labeled with a feasible grasp at each step may bring bias into the policy. Instead, our grasp policy is trained through reinforcement learning guided by task-level rewards, avoiding the drawback of per-step supervision. The policy and critic MLPs both take $L$ cross-attention features as input, and output logits and Q values of $L$ grasp poses. Note that our policy and critic are capable of processing a variable number of actions by parallel processing $L$ cross-attention features, which corresponds to $L$ grasps. We use a task-level reward function as: 
\begin{equation}
\!\!R_{\mathcal{G}}\!=\!\left\{\begin{array}{ll}
1, &\!\text{if\;placing\;the\;object\;to\;its\;desired\;pose} \\
0, &\!\text{if\;placing\;the\;object\;to\;the\;buffer} \\
\!\!\!-1.5, &\!\text{if\;re-placing\;object\;that\;is\;at\;desired\;pose}
\end{array}\right.\\\!\!\!\!\!\!
\end{equation}
When the object is placed in the buffer, the robot gets no reward, which implicitly hinders the grasp of objects whose matched goal is occupied. By guiding the grasp policy with task-level rewards, we consider the task-level efficiency, and can implicitly learn the multimodality of the optimal policy.

\subsection{Place Policy} 
{
When the object is confidently matched, the robot places it to its desired configuration. Given the self-confident matching from $\pi_{\mathcal{S}}$, the place policy $\bar{\pi}_{\mathcal{P}}$ follows Eqn.~\ref{eqn-pi1-place}. {The buffer region includes pixels that are currently unoccupied and do not overlap with any goal regions.} If the grasped object needs to be moved to the buffer, a pixel will be sampled from the region of the buffer for object placement. 

In simulation, the robot has access to the privileged data including the current and goal poses of each object. Considering that we focus more on the learning of $\pi_{\mathcal{G}}$ and $\pi_{\mathcal{S}}$ in this work, and $\bar{ \pi}_{\mathcal{P}}$ brings the same effect to all policies, for simplicity, we calculate the ground-truth 6-DoF pose transformations between the current and goal poses of each object. Note that the goal pose is the pose of the predicted matched goal object. In the real world, object placement is a relative pose estimation problem studied by previous works~\cite{goyal2022ifor,tang2023selective}. This process is implemented by off-the-shelf techniques in our system. We first extract feature correspondence points between the grasped object and its matched goal object by SuperGlue~\cite{sarlin2020superglue}, then predict transformation by RANSAC registration. {Here we use SuperGlue for pose estimation instead of optical flow, as it extracts better correspondences when two objects are successfully matched.} Sometimes SuperGlue\!~\cite{sarlin2020superglue} may fail to extract enough~({\it i.e.} $\geq$5) correspondence points. In these cases, we heuristically set several gripper orientations facing downward or forward for choices, and the one nearest to the final gripper pose is executed.
} 

{
\subsection{Implementation Details of Policy Learning}
\label{sec:grasp-train}
}
Our object models are from GraspNet-1Billion, and object bounding boxes are generated the same way as Sec.~\ref{as_details}. Note that this representation is not a ground-truth bounding box, but considers object occlusion. 

{
{\bf Behavior Cloning Pre-training.} To ease the training of the grasp policy, we first pre-train the policy by behavior cloning. 

We collect 560 sequences of demonstration data. The demonstration sequences are generated by a rule-based planner similar to IFOR~\cite{goyal2022ifor}, which first removes the non-goal objects, then grasps objects that can be directly rearranged to their goal poses, and if none of the objects can be directly rearranged to the goals, randomly grasps objects to the free buffer space in the workspace. For each step through the demonstration sequence, the grasp action selected by the rule-based planner is labeled as 1, with other candidate grasps labeled as 0. Note that during the collection of demonstration data, the agent always has access to the ground-truth object matching and grasping capability. 

By learning from rule-based demonstration data following the optimal policy proved in Sec.~\ref{sec:opt-ideal}, we guide the policy to greedily grasp objects that can be directly placed to the goals, and then resolve the remaining objects with the aid of buffer. In this stage, the policy is trained with cross-entropy loss and AdamW optimizer for 200 epochs. We use the step scheduler for learning rate tuning, which is initially at $1\times 10^{-3}$, with the decay ratio as 0.4 at the step size of 40 epochs.

{\bf Reinforcement Learning.} Based on the behavior cloning pre-training, the grasp policy is further trained with a reinforcement learning stage. 

For each episode, 5$\sim$7 objects are randomly placed into the workspace as the goal scene. Then 6-DoF pose transformations are deployed to these objects with 5 other objects added as non-goal objects to form the initial scene. If all the goal objects are placed at the desired poses and the non-goal objects are placed outside, the episode is regarded as completed. If the planning step exceeds 30, the episode ends as a failure. 

Since the critic network is not trained through behavior cloning, directly training the networks from scratch may bring wrong guidance of exploration from the critic loss. As a result, we first fix the network parameters of the policy MLP as well as the feature encoder to train the critic MLP for 8k iterations as pre-training. Then the policy and critic networks are trained through reinforcement learning for 1.5k episodes. More details can be seen in Appendix~D~\cite{appendix}.
} 

{
\section{Experiments of See}
}
In this section, we carry out a series of experiments to evaluate the {see} policy proposed in Sec.~\ref{active_seeing}. The goals of the experiments are: 1) to validate our confidence-based self-termination mechanism; 2) to demonstrate the effectiveness of our {see} policy in improving object matching under perception noise; and 3) to evaluate the generalization performance of our policy on unseen objects. 

\subsection{Experimental Setup}
{\bf Test Settings.} Our testing cases are collected in PyBullet with the same case generation procedure in Sec.~\ref{sec:as-data-collection}, employing a different random seed. {Example visualizations of the environmental setup are shown in Fig.~\ref{fig:as-cases}.} 

{{\bf Data Collection.}} For each case, the robot grasps the object and rotates it to a new view based on the current object matching result. To highlight the improvement over the initial object matching, we further collect cases where CLIP initially struggles to provide the correct match. Cases are generated with seen and unseen objects from GraspNet-1Billion (48 object models as seen ones, and 24 as unseen ones). {Overall, there are 300 testing cases divided according to different standards. For CLIP-based methods, the testing cases are divided into 4 testing settings: seen normal~(SN), seen hard~(SH), unseen normal~(UN) and unseen hard~(UH). ``Normal'' represents the cases collected in the actual applications, while ``hard'' refers to cases where CLIP frequently fails initially. For comparison to non-CLIP-based methods, the testing cases are divided according to task difficulties into 4 testing settings: less clutter~(LC), heavy clutter~(HC), non-similar appearance~(NA) and similar appearance~(SA). LC and HC denote the clutter level of the goal scene, while SA and NA indicate whether there are objects with similar visual appearances. Example testing cases are visualized in Fig.~\ref{fig:as-cases} and Appendix~E~\cite{appendix}.} For fair comparison, all methods execute the same grasp actions for the same testing cases. 

{\bf Evaluation Metrics.} We evaluate the methods with a series of test settings. Each test setting is measured with 2 metrics:
\begin{itemize}
    \item {\bf Match Success}: the average percentage of match success rate over all test runs. If the robot gets correct object matching within 5 action attempts in a run, then the task is considered successful and completed.
    \item {{\bf See Steps}}: the average motion number of active {perception} per task completion.
\end{itemize}

\begin{figure}[t]
  \centering
  \includegraphics[width=\linewidth]{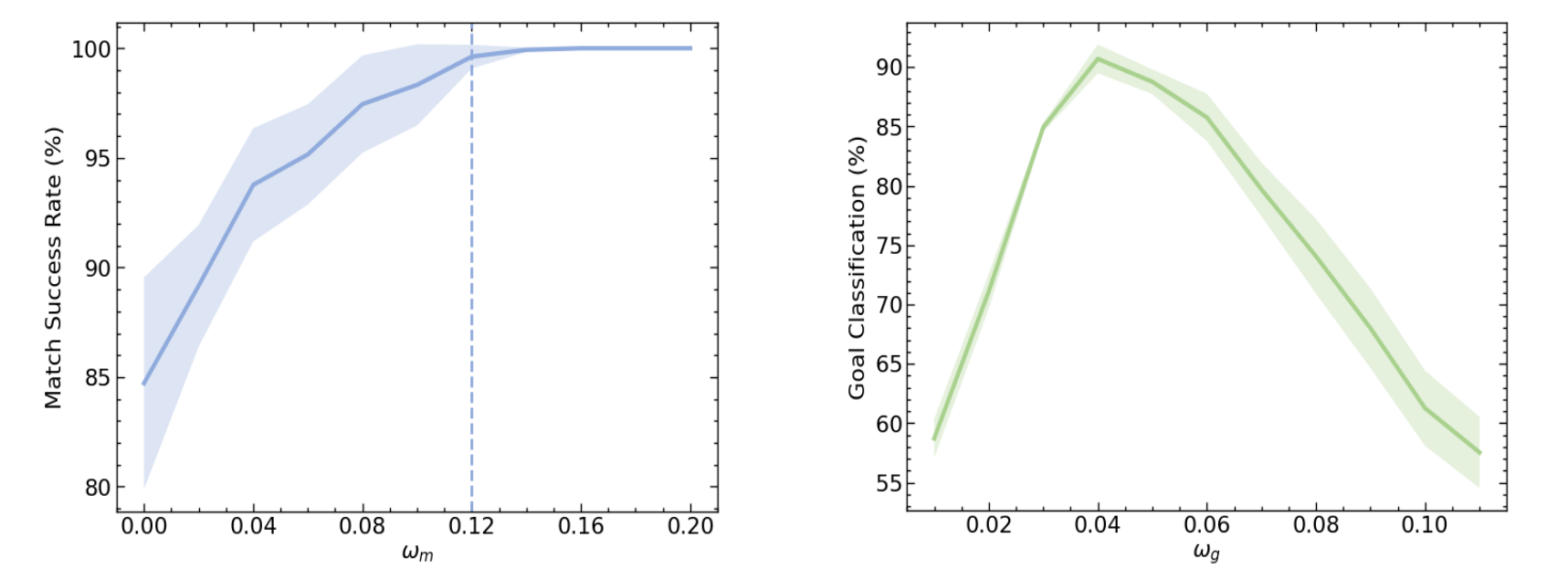}
  \vspace{-0.75cm}
  \caption{Match success rate (left) across different confidence offsets upon the uniform distribution. We choose $\omega_m=0.12$ for the confidence-based terminal signal. Goal classification (right) across different confidence offsets upon the uniform distribution. We choose $\omega_g=0.04$ for the boundary of goal and non-goal objects.}
  \label{fig:terminal-signal}
  \vspace{-0.45cm}
\end{figure}
\begin{figure}[t]
  \centering
  \includegraphics[width=\linewidth]{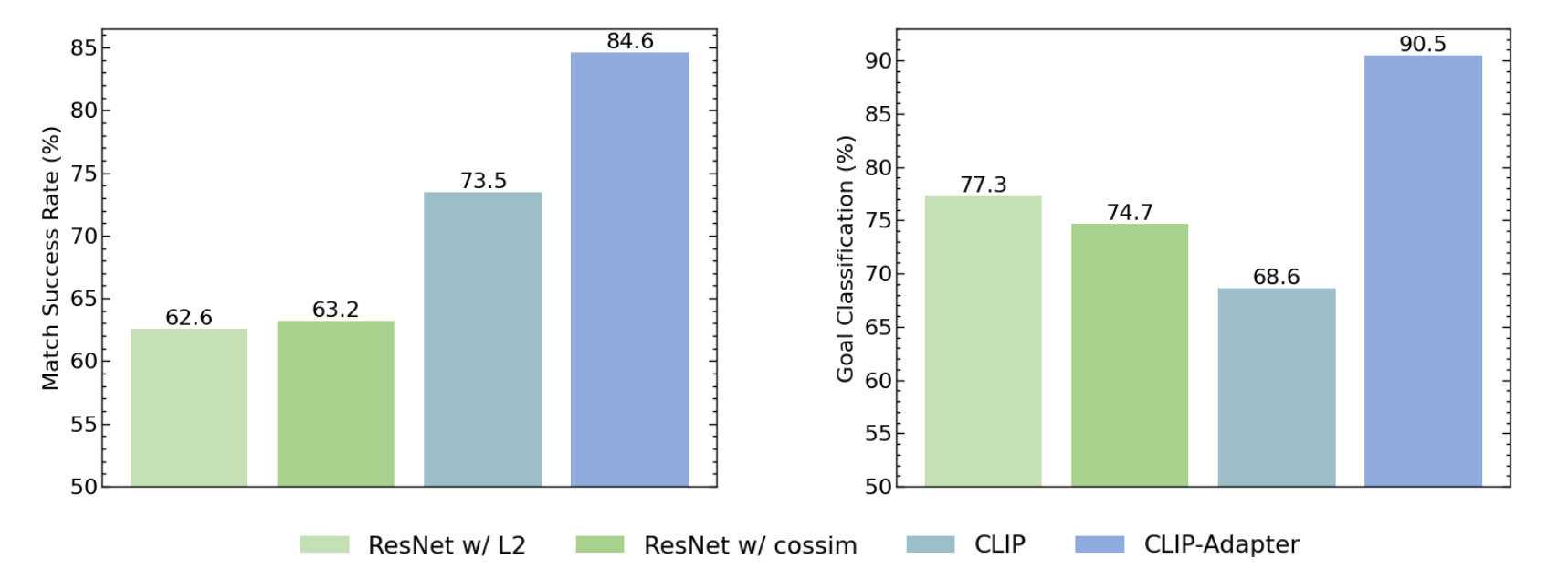}
  \vspace{-0.75cm}
  \caption{Comparisons in match success rate (left) and goal classification (right) of matchers.}
  \vspace{-0.55cm}
  \label{fig:matcher}
\end{figure}
\subsection{Validation of Confidence-based Terminal Signal}
\label{sec:ts-cb-details}
Our first experiment verifies the confidence-based terminal signal of the finetuned CLIP. We define the confident threshold upon the uniform distribution of $\frac{1}{N}$, where $N$ is the number of goal objects. We collect 3k samples of {6-DoF pose transformations} including 1.5k data with goal objects and 1.5k data with non-goal objects. For goal objects, we measure the match success rate under different confidence-based terminal signals of $\zeta_m=\frac{1}{N}+\omega_m$, where $\omega_m$ is the confidence offset upon the uniform distribution. For each $\omega_m$, we calculate the match success rate of data whose matching score is lower than $\zeta_m$. Results are shown in Fig.~\ref{fig:terminal-signal}, from which we can see that when $\omega_m=0.12$, the match success rate is around 1 with a small standard deviation. For all the data of goal and non-goal objects, we measure the goal classification rate with different $\zeta_g=\frac{1}{N}+\omega_g$ as the boundary of goal and non-goal objects. Goal classification is the accuracy of distinguishing goal and non-goal objects. From Fig.~\ref{fig:terminal-signal}, we choose $\omega_g=0.04$ for the boundary of goal and non-goal objects.

\vspace{-0.3cm}
\subsection{Comparison of Matcher}
\label{sec:matcher-comparison}
We compare different matchers to test their object matching performances: 1) {\bf ResNet w/ L2}~\cite{qureshi2021nerp} leverages ImageNet pre-trained ResNet50 to extract latent features and computes L2 norm based object similarity scores; 2) {\bf ResNet w/ cossim}~\cite{tang2023selective} deploys ImageNet pre-trained ResNet50 but calculates the cosine similarity for matching; 3) {\bf CLIP} and 4) {\bf CLIP-Adapter}. {We calculate separate confidence thresholds for different matchers in the same way in Sec.~\ref{sec:ts-cb-details}.}

We collect 3k samples of {6-DoF pose transformations} which are unseen during CLIP finetuning. For goal objects, we measure the match success rate. Note that if the goal object is misclassified as a non-goal object, it is regarded as a failure matching. For all objects, we measure the goal classification rate. Compared results are reported in Fig.~\ref{fig:matcher}. We can observe that ResNet50 with different distance metrics have similar performances, demonstrating poor match success rate in cases of large {pose transformations}. This may be because ResNet50 concerns the local appearance of the object. Instead, with the semantic priors, CLIP cares more about the object-level information, and thus can identify objects of the same category with large {pose transformations}. However, it is interesting to note that CLIP struggles to distinguish goal or non-goal objects with a fixed threshold. By adapter finetuning, CLIP-Adapter shows much better goal classification accuracy and match success rate.

\begin{table}[t]
\caption{Simulation Results of See on Normal And Hard Cases}
\label{table:simulated-as-normal}
\vspace{-0.5cm}
\begin{center}
\resizebox{\columnwidth}{!}{
\begin{tabular}{cccccccccccc}
\hline 
Method & TS & L & D & \multicolumn{4}{c}{Match Success} & \multicolumn{4}{c}{See Steps}\\
\hline 
 & & & & SN & UN & SH & UH & SN & UN & SH & UH\\
\hline
CLIP $\spadesuit$ & & & & \makecell[c]{62.0} & \makecell[c]{48.0} & \makecell[c]{29.0} & \makecell[c]{26.0} & \makecell[c]{--} & \makecell[c]{--} & \makecell[c]{--} & \makecell[c]{--}\\
CLIP-Adapter $\spadesuit$ & & & & \makecell[c]{76.0} & \makecell[c]{50.0} & \makecell[c]{57.0} & \makecell[c]{29.0} & \makecell[c]{--} & \makecell[c]{--} & \makecell[c]{--} & \makecell[c]{--}\\
\hline
Random-CLIP & GT & \XSolidBrush & \XSolidBrush & \makecell[c]{77.3} & \makecell[c]{66.7} & \makecell[c]{55.1} & \makecell[c]{47.1} & \makecell[c]{{\bf 1.50}} & \makecell[c]{2.13} & \makecell[c]{1.51} & \makecell[c]{{\bf 1.73}}\\
Random-Adapter & GT & \XSolidBrush & \XSolidBrush & \makecell[c]{90.5} & \makecell[c]{69.6} & \makecell[c]{73.6} & \makecell[c]{51.7} & \makecell[c]{2.00} & \makecell[c]{{\bf 1.14}} & \makecell[c]{{\bf 1.33}} & \makecell[c]{2.10}\\
Sparse-CLIP & GT & \Checkmark & \XSolidBrush & \makecell[c]{80.4} & \makecell[c]{73.2} & \makecell[c]{55.1} & \makecell[c]{46.3} & \makecell[c]{1.80} & \makecell[c]{2.14} & \makecell[c]{1.86} & \makecell[c]{2.40}\\
Sparse-Adapter & GT & \Checkmark & \XSolidBrush & \makecell[c]{90.9} & \makecell[c]{71.4} & \makecell[c]{76.1} & \makecell[c]{58.6} & \makecell[c]{2.17} & \makecell[c]{1.86} & \makecell[c]{2.25} & \makecell[c]{2.07} \\
Dense-CLIP & GT & \Checkmark & \Checkmark & \makecell[c]{{88.7}}  & \makecell[c]{\bf 84.4} & \makecell[c]{69.5} & \makecell[c]{{\bf 64.8}} & \makecell[c]{{1.75}}& \makecell[c]{1.71} & \makecell[c]{{1.61}}& \makecell[c]{1.94} \\
Dense-Adapter & GT & \Checkmark & \Checkmark & \makecell[c]{{\bf 95.3}}  & \makecell[c]{80.0} & \makecell[c]{{\bf 84.4}}  & \makecell[c]{63.1} & \makecell[c]{{2.00}}& \makecell[c]{1.43} & \makecell[c]{{2.04}}& \makecell[c]{1.89}\\
\hline
Random-CLIP & CB & \XSolidBrush & \XSolidBrush & \makecell[c]{68.9} & \makecell[c]{48.8} & \makecell[c]{39.9} & \makecell[c]{21.8} & \makecell[c]{2.63} & \makecell[c]{3.67} & \makecell[c]{{\bf 2.35}} & \makecell[c]{2.71} \\
Random-Adapter & CB & \XSolidBrush & \XSolidBrush & \makecell[c]{78.6} & \makecell[c]{54.6} & \makecell[c]{59.8} & \makecell[c]{34.7} & \makecell[c]{3.83} & \makecell[c]{{2.93}} & \makecell[c]{2.45} & \makecell[c]{3.40}\\
Sparse-CLIP & CB & \Checkmark & \XSolidBrush & \makecell[c]{65.9} & \makecell[c]{53.5} & \makecell[c]{41.3} & \makecell[c]{23.9} & \makecell[c]{3.25} & \makecell[c]{{\bf 2.55}} & \makecell[c]{2.61} & \makecell[c]{{\bf 2.36}}\\
Sparse-Adapter & CB & \Checkmark & \XSolidBrush & \makecell[c]{77.3} & \makecell[c]{56.1} & \makecell[c]{64.7} & \makecell[c]{33.3}  & \makecell[c]{2.86} & \makecell[c]{2.79} & \makecell[c]{2.62} & \makecell[c]{3.38}\\
Dense-CLIP & CB & \Checkmark & \Checkmark & \makecell[c]{75.0}  & \makecell[c]{68.1} & \makecell[c]{{63.9}}  & \makecell[c]{29.0} & \makecell[c]{\bf {2.46}}& \makecell[c]{3.20} & \makecell[c]{2.38}& \makecell[c]{{2.45}}\\
Dense-Adapter & CB & \Checkmark & \Checkmark & \makecell[c]{{\bf 85.7}}  & \makecell[c]{{\bf 68.3}} & \makecell[c]{{\bf 69.0}} & \makecell[c]{{\bf 39.7}} & \makecell[c]{{2.62}}& \makecell[c]{3.32} & \makecell[c]{{2.64}}& \makecell[c]{3.36}\\
\hline
\end{tabular}
}
\begin{tablenotes}
\scriptsize
\item * Policies marked with $\spadesuit$ do not use active perception. TS represents Terminal Signal. L and D denote whether learned and using dense reward respectively. GT and CB denote ground-truth and confidence-based terminal signals respectively.
\end{tablenotes}
\end{center}
\vspace{-0.65cm}
\end{table}

\begin{table}[t]
\caption{Simulation Results of See on Different Task Difficulties}
\label{table:simulated-as-difficulties}
\vspace{-0.5cm}
\begin{center}
\resizebox{\columnwidth}{!}{
\begin{tabular}{ccccccccccc}
\hline 
Method & TS & \multicolumn{4}{c}{Match Success} & \multicolumn{4}{c}{See Steps} \\
\hline 
 & & LC & HC & NA & SA & LC & HC & NA & SA \\
\hline
Dense-LoFTR & GT & \makecell[c]{61.9} & \makecell[c]{45.5} & \makecell[c]{61.6} & \makecell[c]{59.1} & 1.92 & 1.60 & 2.04 & 2.14\\
Dense-R3M & GT & \makecell[c]{78.8} & \makecell[c]{61.5} & \makecell[c]{82.0} & \makecell[c]{63.3} & 1.52 & 2.20 & 1.44 & 2.22\\
Dense-MVP & GT & \makecell[c]{88.1} & \makecell[c]{{\bf 87.0}} & \makecell[c]{90.5} & \makecell[c]{70.2} & {\bf 1.27} & {\bf 1.20} & {\bf 1.08} & {\bf 1.11}\\
Dense-Adapter & GT & \makecell[c]{{\bf 92.7}}  & \makecell[c]{84.0} & \makecell[c]{{\bf 92.8}}& \makecell[c]{{\bf 83.3}} & 1.68 & 1.71 & 2.00 & 1.90\\
\hline
Dense-LoFTR & CB & \makecell[c]{32.6} & \makecell[c]{27.6} & \makecell[c]{27.6} & \makecell[c]{34.7} & 2.45& {\bf 1.00}& 1.89 & 3.00\\
Dense-R3M & CB & \makecell[c]{48.9} & \makecell[c]{48.0} & \makecell[c]{59.0} & \makecell[c]{41.2} & {\bf 2.42} & 2.00 & {\bf 1.73} & {\bf 1.67}\\
Dense-MVP & CB & \makecell[c]{66.4} & \makecell[c]{62.5} & \makecell[c]{65.7} & \makecell[c]{52.2} & 3.08& 3.50& 2.40& 3.92\\
Dense-Adapter & CB & \makecell[c]{{\bf 70.5}}  & \makecell[c]{{\bf 65.0}} & \makecell[c]{{\bf 80.0}}& \makecell[c]{{\bf 65.0}} & 3.95 & 3.60 & 3.50 & 3.68\\
\hline
\end{tabular}
}
\begin{tablenotes}
\scriptsize
\item * TS represents Terminal Signal. All the policies are learned with dense rewards. GT and CB denote ground-truth and confidence-based terminal signals respectively.
\end{tablenotes}
\end{center}
\vspace{-0.75cm}
\end{table}

\subsection{Ablation Studies}
We compare our methods with a series of ablation methods on all test settings to demonstrate: 1) whether CLIP finetuning endows better object matching; 2) whether our learning-based policy is more efficient than random exploration; 3) whether our dense reward mechanism is effective; 4) whether our terminal signal is efficacious for {the see policy}; 5) whether our policy can generalize to unseen objects; 6) whether semantic-based matcher ({\it i.e.} CLIP, CLIP-Adapter) brings benefits to our policy. For fair comparison, {all learning-based policies are trained with the same scheme as in \ref{as_details}, including environment setup, RL training process and RL algorithm details, except for using different rewards or visual encoders.}

{\bf CLIP versus CLIP-Adapter.} We plug policies into CLIP and CLIP-Adapter to measure the effectiveness of CLIP finetuning. We can see from Table~\ref{table:simulated-as-normal} that all of the policies with finetuned CLIP outperform those with pre-trained CLIP by 6$\%$$\sim$15$\%$ match success rate in cases of {SN}, and show comparable or better match performance in cases of {UN}. The performance gain is more obvious in the hard cases. This suggests the advantages of the lightweight adapter in improving the matching performance without losing the generalization capability.

{\bf Random Exploration versus Learning-based Policies.} We also test policies that randomly rotate the in-hand object to the next view for matching~({Random-CLIP}, {Random-Adapter}) to validate the efficiency of our learning-based policies~({Dense-CLIP}, {Dense-Adapter}). {In Table~\ref{table:simulated-as-normal}, ``L'' denotes whether the policy is learned}. Results demonstrate that our methods achieve higher match success rates, indicating the effectiveness and efficiency of our learning-based policies. While random policies can get better performance upon the sole matchers with arbitrary exploration, they still fail to defeat our policies using the same matchers. Besides, although random policies achieve the least {see steps} in two settings, it does not mean that random strategy has high efficiency, as the corresponding match success rates are lower than our policies by around 11$\%$.

{\bf Sparse Reward versus Dense Reward.} We further make comparisons between our policies and ablation policies that deploy sparse reward without the guidance of delta entropy and penalty~({Sparse-CLIP}, {Sparse-Adapter}). {Testing results in Table~\ref{table:simulated-as-normal} (``D'' indicates whether the policy is learned with dense rewards)} show the advantage of the guided dense rewards, which simultaneously improves the match success rate and the {see} efficiency. In contrast, {Sparse-CLIP} and {Sparse-Adapter} require more {see steps} to achieve a good matching view.

{\bf Ground-truth~(GT) versus Confidence-based~(CB) Terminal Signals{~(TS)}.} We validate the terminal signals of the {see} policy. We first test the methods with ground-truth terminal signals that indicate whether the grasped object is correctly matched. Results shown in Table~\ref{table:simulated-as-normal} demonstrate that our methods can achieve the highest match success rate with around 2 {see steps}. However, in actual application, there is no ground-truth terminal signal. Therefore, the policy is supposed to terminate the {see} process by itself. We define confidence-based terminal signals for {see} in Sec.~\ref{sec:as_policy}. Overall, without ground-truth terminal signals, all of the policies get lower match success rates and cost more {see steps}. Nonetheless, {Dense-Adapter} consistently outperforms all other methods in match success rate across all testing settings, with a little bit more {see steps}. Also, by learning with dense rewards which encourages the policy to improve matching confidence as soon as possible, our policies display {fewer} performance drops in both match success rate and efficiency, further confirming the advantage of our dense reward mechanism. Instead, the random strategy struggles to identify an effective next view to reduce the matching uncertainty, and often executes failed in-hand rotation {\it e.g.} collision between the object and the arm.

{\bf Seen Objects versus Unseen Objects.} {All CLIP-based methods} are tested with both seen objects and unseen objects to validate the generalization performance. Overall, our policy (Dense-Adapter) achieves the highest match success rate across all testing settings with seen and unseen objects, verifying the object generalization capability of our method. This may benefit from CLIP and our self-confidence strategy. Through a balanced adapter finetuned mechanism, we inherit the zero-shot generalization from CLIP. In addition, we guide the policy learning with CLIP matching uncertainty to get self-confident matching, which follows the CLIP distribution, thus preserving the generalization ability to some extent. It is interesting to note from Table~\ref{table:simulated-as-normal} that when testing with unseen objects and ground-truth terminal signals, Dense-Adapter shows a slightly lower match success rate than Dense-CLIP, but costs {fewer see steps}. This may be due to the fact that the mixture of adapted features descends the generalization to some extent. However, with confidence-based terminal signals, Dense-Adapter gets higher match success rates than Dense-CLIP. We attribute this performance to our CLIP finetuning scheme. {By finetuning CLIP with the one-hot encoded labels of the ground-truth matched objects, we implicitly bring confidence prior into the matcher. }

{\bf Normal Cases versus Hard Cases.} We test all {CLIP-based methods} with hard cases where CLIP fails most of the time. Generally, the performance gain from CLIP to policies is more obvious. The performance gap among policies is also more apparent. For example, in normal cases, Random-Adapter can get around 90$\%$ match success rate with seen objects and ground-truth terminal signals. However, in hard cases, it struggles to find an effective view through random exploration, resulting in a large performance drop. Instead, our policy can still achieve around 85$\%$ match success rate with {fewer see steps} compared to random policies and policies with sparse rewards.

{
{\bf Different Visual Matchers.} We apply ablation studies with three different visual matchers {\it i.e.} LoFTR~\cite{sun2021loftr}, R3M~\cite{nair2022r3m} and MVP~\cite{xiao2022masked}, namely {Dense-LoFTR}, Dense-R3M and Dense-MVP respectively. LoFTR~\cite{sun2021loftr} is a transformer-based image matching method, while R3M~\cite{nair2022r3m} and MVP~\cite{xiao2022masked} are visual representation models pre-trained with human video data for robotics tasks. Comparison results on different task difficulties are reported in Table.~\ref{table:simulated-as-difficulties}. Overall, Dense-Adapter achieves higher match success rates across different task difficulties. 

Dense-LoFTR exhibits the worst performance. This may be due to the fact that LoFTR focuses on pixel-level information, and is not concerned about object-level information. Additionally, the object-centric image crop, characterized by low resolution and little texture, poses a challenge for LoFTR. By pre-training with human video data, Dense-R3M and Dense-MVP get better awareness of objects, thus gaining better performance. In particular, Dense-MVP shows comparable results with Dense-Adapter on the task difficulty settings of LC, HC and NA with GT terminal signal. 
However, Dense-MVP struggles to distinguish objects with similar visual appearances, making it hard to self-terminate the see process with the correct results. Conversely, finetuning CLIP using one-hot encoded labels incorporates contrastive priors into the visual representation, prompting our policy to obtain better matching results. Compared to Dense-LoFTR, our policy leverages the semantic priors from CLIP to recognize the same objects despite large pose transformations.
}

\vspace{-0.3cm}
\subsection{Case Studies}
\begin{figure*}[t]
  \centering
  \includegraphics[width=\linewidth]{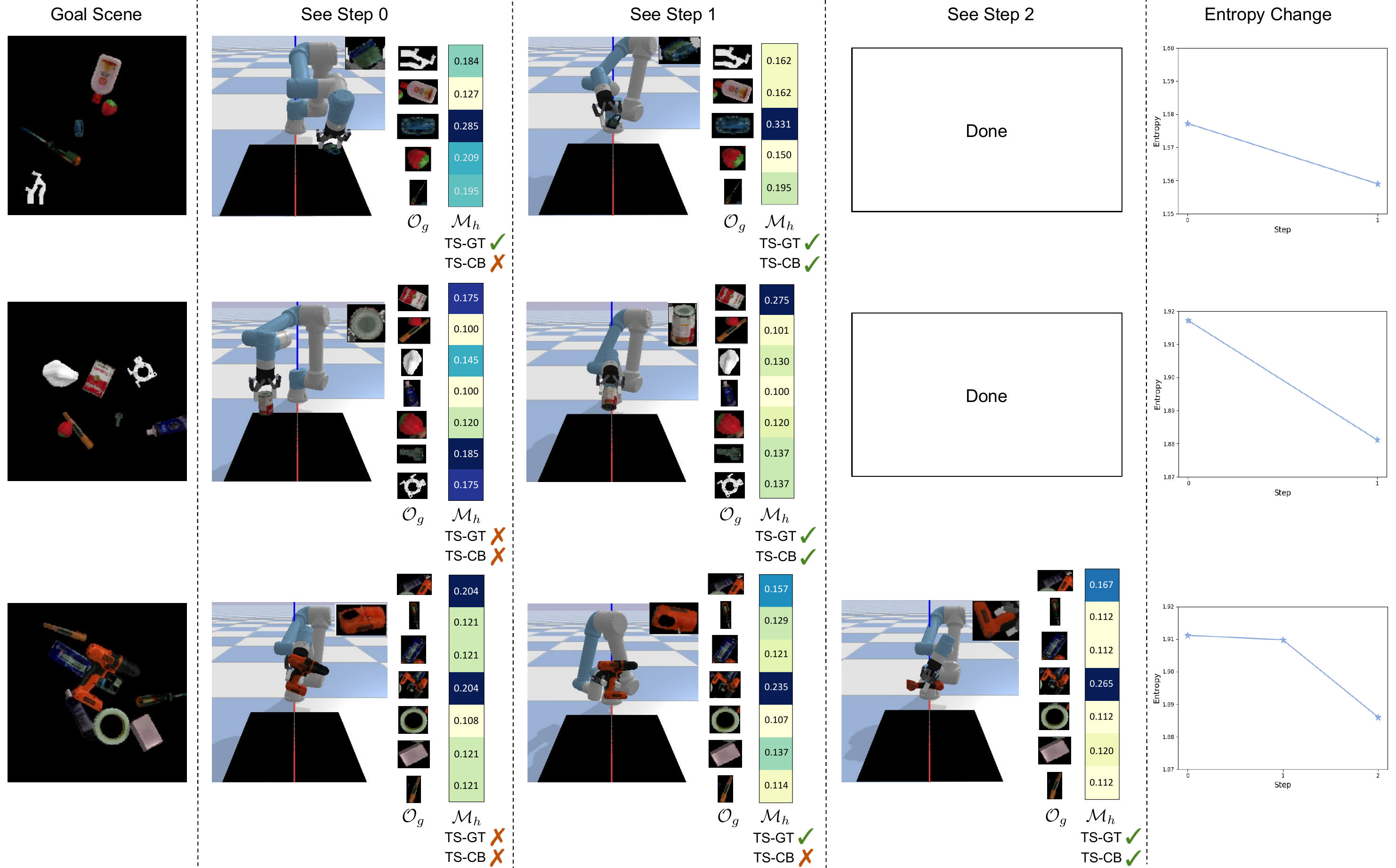}
  \vspace{-0.7cm}
  \caption{{Case studies of the see policy. For each case, we show the see process in a third-person view and camera view (object-centric), and the corresponding matching distribution to the goal objects is shown beside. For each see step, we mark the status of GT terminal signal and CB terminal signal, where \Checkmark represents reaching the terminal signal and \XSolidBrush otherwise. Also, the entropy change of the whole see process is paired with each case. The first two cases are of the easy category, and the third is of the hard category.}}
  \label{fig:as-cases}
  \vspace{-0.6cm}
\end{figure*}

We show three typical cases of {see} in Fig.~\ref{fig:as-cases}. In the first case, the matching is initially correct but does not activate the confidence-based terminal signal. Thus, active {perception} is required to improve matching confidence. In the second case, the ground truth terminal signal and the confidence-based terminal signal are activated at the same time, which also suggests the effectiveness of the {see} policy to rectify the wrong matching. Most cases align with the third one, where the {see} policy first corrects the matching, and then improves the matching confidence for self-termination. Overall, the matching entropy can be reduced via active {perception}.

\section{Experiments of Object Rearrangement}
In this section, we evaluate the application of our object rearrangement system. The goals of the experiments are: 1) to indicate the efficiency of the grasp policy; 2) to demonstrate that our system can achieve better task-level performance for object rearrangement with the challenge of clutter, swap, selectivity and {6-DoF pose transformations}; 3) to evaluate the generalization of our system to unseen objects; 4) to test whether our system can successfully transfer to the real world. 

\subsection{Experimental Setup}
{\bf Test Settings.} In the following experiments, we evaluate our methods with a series of testing settings shown in Table~\ref{table:simulated-rearrangement}. There are 200 testing cases divided into different settings. 

{{\bf Data Collection.}} For each case, 4$\sim$8 objects are placed in the workspace to form the goal configuration, captured as the goal image. Then 3-DoF or 6-DoF pose transformations are applied to the goal objects while 0$\sim$5 non-goal objects are randomly dropped into the workspace, resulting in the initial scene. Cases involving non-goal objects are regarded as selective ones. All testing cases require swaps, where the goal positions of some objects are initially occupied by others, and initially constitute clutter scenes. Each setting is tested with seen and unseen objects. {Example visualizations of testing cases are shown in Fig.~\ref{fig:or-cases} and Appendix~E~\cite{appendix}.}

{\bf Evaluation Metrics.} We evaluate the methods with a series of test settings, which are measured with 5 metrics:
\begin{itemize}
    \item {\bf Task Completion}: the average percentage of completion rate over all test runs. If the robot places all goal objects to their goal position within 5cm error and places all non-goal objects outside in $b$ pick-n-place steps, then the task is considered successful and completed.
    \item {{\bf Completion Planning Steps}}: the average pick-n-place number over all completed test runs. 
    \item {{\bf Overall Planning Steps}}: the average pick-n-place number over all test runs. If the task is completed, the planning step is the pick-n-place number to complete the task. Otherwise, the planning step is recorded as the max planning step $b$.
    \item {\bf Pos Error}: the average position error~($cm$) of all the goal objects per task completion. Position error is defined as the Euclidean distance between the goal arrangement and the achieved arrangement.
    \item {\bf ADD-S}: the average ADD-S metric~($10^{-2}$, {the average closest point distances of model points~\cite{hinterstoisser2013model}}) of all the goal objects per task completion, where $X_a^j$ and $X_g^j$ denote the point of the achieved and goal point clouds of $o_g^j$ respectively.
    \begin{equation}
    \!\!\!\text{ADD-S}=\frac{1}{N}\frac{1}{|X_g^j|}\sum_{j}\sum_{X_g^j}\min_{X_a^j}\left\|X_g^j-X_a^j\right\|
    \end{equation}
     
\end{itemize}

{All the metrics expect {\bf Overall Planning Steps} and {\bf ADD-S} are consistent with those in \cite{goyal2022ifor,tang2023selective}. {\bf Completion Planning Steps} only considers the efficiency of completed test runs. However, if the completion rate is low, then fewer planning steps does not mean the corresponding policy is more efficient. Thereby we have another metric {\bf Overall Planning Steps} to consider the overall efficiency of a policy. {\bf ADD-S} is an additional metric to evaluate the 6-DoF pose error of the rearrangement task. {\bf Pos Error} and {\bf ADD-S} are only evaluated in the real world. This is because the robot has access to the privileged data of object poses in simulation. In real-world experiments, object placement is implemented by off-the-shelf relative pose estimation to make the evaluation closed-loop. {\bf Pos Error} and {\bf ADD-S} are calculated after each inference as feedback of whether the task has been completed.}

\subsection{Baseline Implementations}
We compare the performance of our system to the following baselines. Note that there is no active {perception} in these two baseline methods.

{
{\bf E2E} is a method that directly takes the images of current and goal scenes as input and outputs the grasp pose. The images of current and goal scenes are embedded by CLIP and fused with cross-attention. Then, the policy predicts continuous grasp actions with a GMM (Gaussian Mixture Model) output head, which has been shown effective to capture the diverse multimodal behaviors~\cite{mandlekar2021matters}. For fair comparison, the training process~(including behavior cloning pre-training and reinforcement learning), the place strategy, as well as the data amount are the same as ours.
}

{\bf IFOR$^{\dag}$} is a method using a rule-based planner, which sorts the relative pose transformations of all objects between the current and the goal scenes, and greedily picks objects with larger relative transformations and free goal positions. As the raw version of IFOR~\cite{goyal2022ifor} cannot handle rearrangement settings with selectivity, we apply our matcher ({\it i.e.} finetuned CLIP) to distinguish non-goal objects. For fair comparison, the graspnet and the place strategy also remain the same as ours.

{\bf SeRe$^{\dag}$} is a method that conducts learning-based graph planning for object sequencing and action selection. The raw version of Selective Rearrangement~\cite{tang2023selective} implements top-down pixel-wise action selection, which is limited in 3-DoF settings. To extend it to 6-DoF settings, we reproduce the method with the graspnet used in our method to predict 6-DoF grasp poses. The matcher and the place strategy are the same as ours.

\subsection{Simulation Experiments}
Our simulation environment is built in PyBullet~\cite{coumans2021}, which involves a UR5 arm, a ROBOTIQ-85 gripper, and a camera of Intel RealSense L515.

\begin{figure}[t]
  \centering
  \includegraphics[width=\linewidth]
  {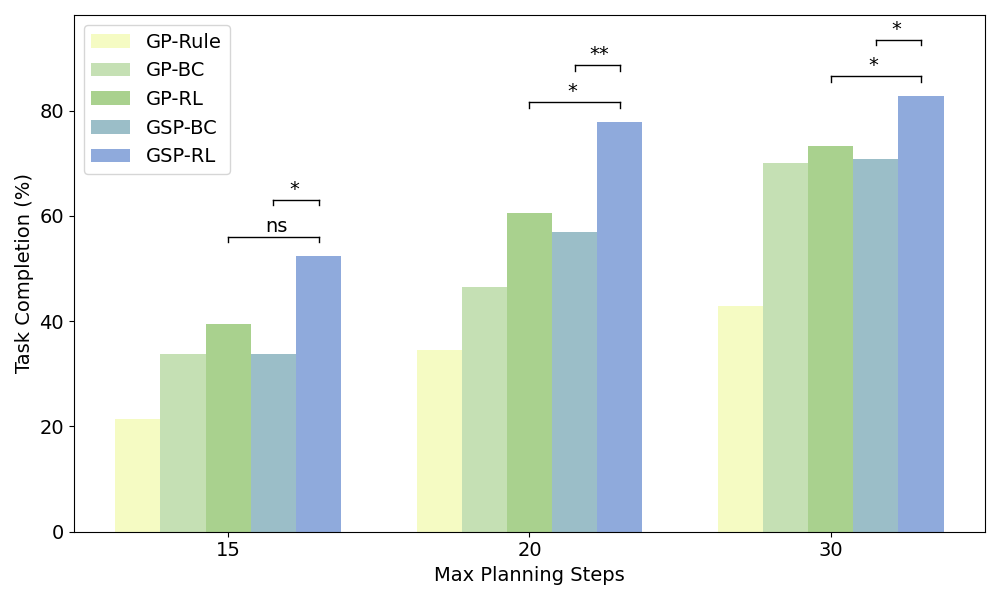}
  \vspace{-0.7cm}
  \caption{Ablation studies in planning efficiency of unknown object rearrangement. We compare GSP(-RL) with four ablation methods with the setting of 6-DoF, selective and seen objects. {Statistical significance is computed by Mann-Whitney U Test where ns means $p\textless0.1$, $\star$ means $p\textless0.05$ and $\star\star$ means $p\textless0.01$.}}
  \label{fig:efficiency}
  \vspace{-0.6cm}
\end{figure}

\begin{table*}[t]
\caption{Simulation Results of Object Rearrangements}
\label{table:simulated-rearrangement}
\vspace{-0.2cm}
\begin{center}
\begin{threeparttable}
\begin{tabular}{ccccccccccccc}
\hline Method & Rotation & \makebox[0.05\textwidth][c]{Selectivity}& Init. \#obj. & Goal \#obj. & \multicolumn{2}{c}{Task Completion} & \multicolumn{2}{c}{Completion Planning Steps} & \multicolumn{2}{c}{Overall Planning Steps}\\
\hline 
 & & & & & seen & unseen & seen & unseen & seen & unseen \\
\hline 
E2E & \multirow{5}{*}{3-DoF} & \multirow{5}{*}{\XSolidBrush} & \multirow{5}{*}{4-8} & \multirow{5}{*}{4-8} & 0.0 & 0.0 & -- & -- & 30.0$\pm$0.0 & 30.0$\pm$0.0 \\
IFOR$^{\dag}$ & & & & & 80.0 & 66.7 & {\bf 10.0$\pm$2.7} & {\bf 9.00$\pm$1.8} & \underline{14.0$\pm$4.7}& 16.0$\pm$4.9\\
SeRe$^{\dag}$ & & & & & 80.0 & 70.0 & 11.3$\pm$2.5 & 10.5$\pm$0.8 & 15.0$\pm$4.8 & 16.4$\pm$4.5\\
GSP-BC & & & & & \underline{88.5} & \underline{76.7} & 14.3$\pm$3.7 & \underline{9.50$\pm$3.6} & 16.1$\pm$4.3& \underline{14.3$\pm$4.4}\\
GSP & & & & & {\bf 92.0} & {\bf 83.3}& \underline{11.1$\pm$2.3} & 11.0$\pm$1.6 & {\bf 12.6$\pm$3.4}& {\bf 14.2$\pm$3.5}\\
\hline
E2E & \multirow{5}{*}{3-DoF} & \multirow{5}{*}{\Checkmark} & \multirow{5}{*}{9-13} & \multirow{5}{*}{4-8} & 0.0 & 0.0 & -- & -- & 30.0$\pm$0.0 & 30.0$\pm$0.0 \\
IFOR$^{\dag}$ & & & & & \underline{88.9} & \underline{66.7} & {\bf 14.4$\pm$1.7} & \underline{13.0$\pm$1.2} & {\bf 16.1$\pm$2.5} & \underline{18.7$\pm$4.0}\\
SeRe$^{\dag}$ & & & & & 83.3 & \underline{66.7}& \underline{14.5$\pm$2.8} &15.8$\pm$1.2 & \underline{17.1$\pm$3.1}& 20.5$\pm$3.4\\
GSP-BC & & & & & 68.8 & \underline{66.7} & 17.2$\pm$2.3 & {\bf 12.5$\pm$1.6} & 21.2$\pm$3.5& {\bf 18.3$\pm$4.1}\\
GSP & & & & & {\bf 91.7} & {\bf 76.7}& 16.8$\pm$2.9 & 19.0$\pm$2.2 & 17.9$\pm$3.3 & 21.6$\pm$3.0\\
\hline
E2E & \multirow{5}{*}{6-DoF} & \multirow{5}{*}{\XSolidBrush} & \multirow{5}{*}{4-8} & \multirow{5}{*}{4-8} & 0.0 & 0.0 & --- & -- & 30.0$\pm$0.0 & 30.0$\pm$0.0 \\
IFOR$^{\dag}$ & & & & & 50.0 & 33.3& \underline{8.00$\pm$0.5} & \underline{8.00$\pm$1.4} & 19.0$\pm$3.9& 22.7$\pm$5.1\\
SeRe$^{\dag}$ & & & & & 41.2 & 30.0& {\bf 7.50$\pm$0.5} & {\bf 6.00$\pm$0.6} & 20.7$\pm$5.5& 22.8$\pm$5.5\\
GSP-BC & & & & & {\bf 94.1} & \underline{50.0} & 12.6$\pm$3.0 & 9.00$\pm$2.6 & {\bf 13.6$\pm$3.6}& \underline{19.5$\pm$5.2}\\
GSP & & & & & \underline{92.9} & {\bf 73.3}& 13.7$\pm$2.9 & 10.5$\pm$2.5 & \underline{14.9$\pm$3.5}& {\bf 15.7$\pm$4.2}\\
\hline
E2E & \multirow{5}{*}{6-DoF} & \multirow{5}{*}{\Checkmark} & \multirow{5}{*}{9-13} & \multirow{5}{*}{4-8}& 0.0 & 0.0 & -- & -- & 30.0$\pm$0.0 & 30.0$\pm$0.0 \\
IFOR$^{\dag}$ & & & & & 33.3 & 33.3 & {\bf 12.0$\pm$1.5} & {\bf 13.0$\pm$1.6} & 24.0$\pm$3.9& 24.3$\pm$4.0\\
SeRe$^{\dag}$ & & & & & 50.0 & 33.3& \underline{12.2$\pm$1.4} & \underline{14.0$\pm$2.2} & 21.1$\pm$4.4& 24.7$\pm$3.8\\
GSP-BC & & & & & \underline{75.0} & \underline{41.7}& 16.1$\pm$2.8 & \underline{14.0$\pm$1.4} &\underline{19.6$\pm$3.8} &\underline{23.3$\pm$4.0}\\
GSP & & & & & {\bf 84.6} & {\bf 70.0}& 14.7$\pm$2.5 & 19.7$\pm$2.4 &{\bf 17.1$\pm$3.6} & {\bf 22.8$\pm$2.8}\\
\hline
\end{tabular}
\begin{tablenotes}
\scriptsize
\item * All cases are with challenges of clutter and swap.
\end{tablenotes}
\end{threeparttable}
\end{center}
\vspace{-0.8cm}
\end{table*}

{\bf Ablation Studies.} In order to evaluate the planning efficiency of GSP(-RL), we compare our method with four variant methods: 1) GP-Rule, a method selecting objects to grasp and place by the rule-based planner used to collect demonstration in Sec.~\ref{sec:grasp-train}, but using predicted object matching and grasp; 2) GP-BC, a method involving grasp and place policies learned with behavior cloning; 3) GP-RL, a method involving grasp and place policies learned with behavior cloning and reinforcement learning; 4) GSP-BC, a method similar to GP-BC, but with {see} policy. We set different limits of planning step ($b\!=\!15, 20, 30$) for the comparison of task completion, and results are presented in Fig.~\ref{fig:efficiency}. {For statistical significance measurement, we conduct Mann-Whitney U Test~\cite{nachar2008mann} between GSP-RL and two ablation methods with the nearest means on 60 samples of tests.} GP-Rule performs the worst across all the methods, which neglects the uncertainty of object matching and grasping. It struggles to complete the task even extending the max planning step to $b=30$. Instead, by jointly considering object matching as well as candidate grasps for grasp selection, GP-BC gets the awareness of perception noise, and learns from demonstration. This improves the performance, especially when extending $b$ from 20 to 30. However, it still performs poorly when $b=15$. This is due to the fact that GP-BC is supervised by the demonstration data labeled as one unique grasp per step, which brings bias into the policy and pays less attention to task-level performance. By further training with reinforcement learning by task-level rewards, GP-RL can effectively improve planning efficiency. When the max step is limited below 20, GP-RL performs much better than GP-BC. Finally, by cooperating with active {perception}, GSP-RL can get better performance. This benefits from more confident object matching, which avoids additional occupancy of other objects' goals, thus reducing the planning steps. Still, we can see from GSP-BC that it also achieves more efficient performance than GP-BC. Overall, GSP-RL is statistically significant~($p\textless0.05$) when $b\geq20$.

Also, we compare GSP and GSP-BC in all settings in Table~\ref{table:simulated-rearrangement} with the max planning step $b=30$. Results indicate that GSP realizes better task completion in all testing settings except for the 6-DoF swap settings of seen objects, which also reports comparable results with {GSP-BC}. This illustrates the effectiveness of reinforcement learning with task-level rewards to improve task completion rate and action efficiency. A noteworthy observation is that {GSP-BC} shows a much worse task completion rate in selectivity settings than in non-selectivity ones. This may be owing to the fact that the demonstration data fails to provide multimodal information {\it e.g.} feasibility of grasps of several objects. In selective settings, such situations are more common, as all the non-goal objects can be directly moved outside. On the contrary, training by reinforcement learning enables the policy to explore more planning sequences beyond the demonstration ones, and implicitly learn a multimodal value function for task-level planning. 

{\bf Comparisons to Baselines.} From compared results in Table~\ref{table:simulated-rearrangement}, we can observe that our system outperforms all baselines on task completion rate. {Results indicate that E2E all fails. This is due to the fact that {\bf E2E} struggles to learn grasp capability in clutter from scratch with limited data, especially for the long-horizon object rearrangement task affected by various perception noises. 
With object-centric representation and the pre-trained graspnet, other methods can achieve better performance.} In 3-DoF settings, there are relatively fewer wrong matches. {IFOR$^{\dag}$} and {SeRe$^{\dag}$} can achieve 80$\%$$\sim$90$\%$ task completion rate within 15 {completion planning steps}, which is consistent with results in \cite{goyal2022ifor,tang2023selective}. However, in 6-DoF settings, these two methods get much worse performance in task completion at 30$\%\sim$50$\%$. Note that although they spend the least planning steps per completion, it may be due to the unfortunate fact that they fail in all cases where the object number is more than the average number of their {completion planning steps}, leading to large {overall planning steps}. For example, in the setting of {6-DoF pose transformations}, selectivity and seen objects, the average number of {completion planning steps} of {IFOR$^{\dag}$} is 12, but the most number of initial objects is 13, which means {IFOR$^{\dag}$} fails in all cases with 13 initial objects. With the skill of active {perception}, our system can get more confident object matching for rearrangement planning, thus achieving much better task completion with 11$\sim$17 {completion planning steps}. Besides, compared to IFOR$^{\dag}$ and SeRe$^{\dag}$ which first choose objects through rule-based planners and then conduct object grasping, GSP evaluates the quality of available grasps by considering perception noise and task-level performance, which also accounts for our better performance.

{\bf Generalization to Unseen Objects.} We test all the approaches with unseen objects. It can be observed from Table~\ref{table:simulated-rearrangement} that GSP can achieve the highest completion rate, demonstrating the generalization capability of our method. Although GSP spends more planning steps per completion, especially in selectivity settings, it does not mean that our method is worse than others. In cases with unseen objects, perception error occurs more frequently than those with seen objects. Thanks to our learning scheme that considers closed-loop perception and planning, our method can handle perception error by spending more planning steps for higher task completion. 
{Generalization on other factors such as camera angle and table background refers to Appendix~F~\cite{appendix}.}

\begin{figure*}[t]
  \centering
  \includegraphics[width=\linewidth]{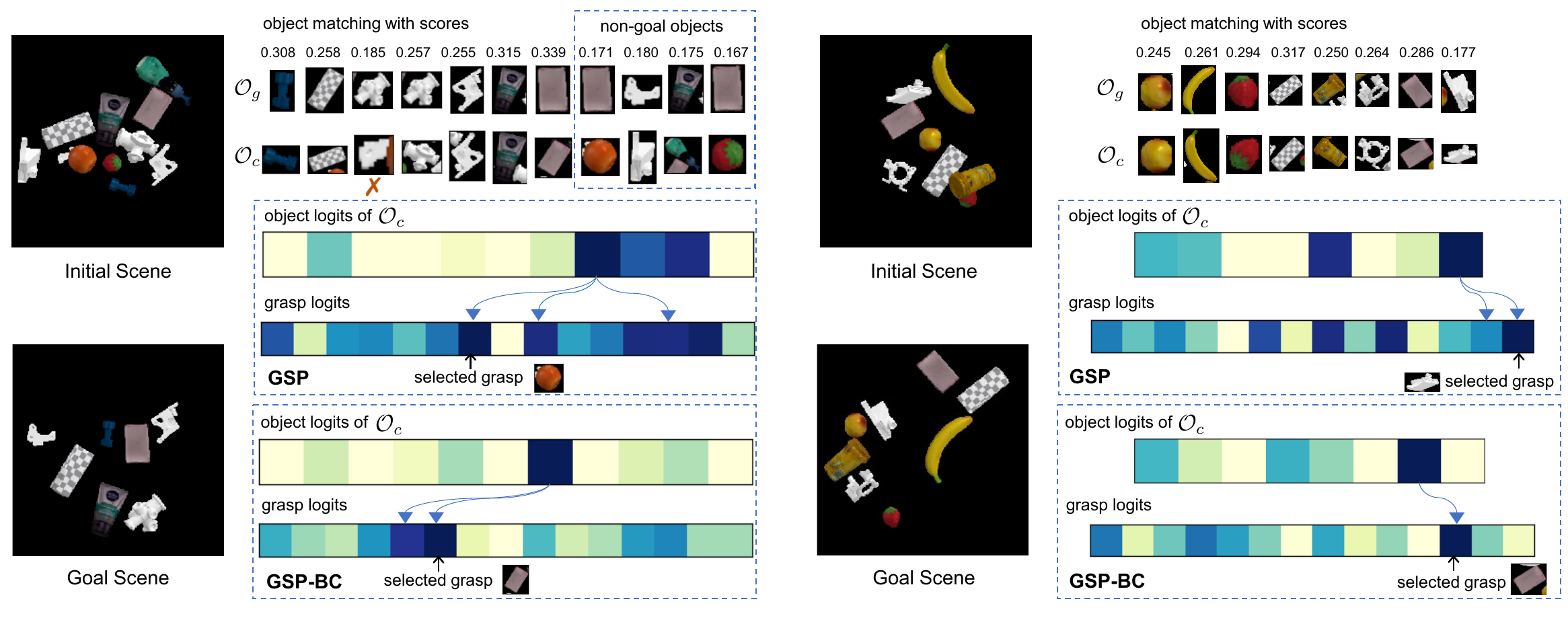}
  \vspace{-0.8cm}
  \caption{Two cases for comparison of the grasp policies. {Both cases are with 6-DoF pose transformations, while only the left case is with selectivity}. Given the initial scene and goal scene for each case, we show the object matching results between each current object and all the goal objects with their confidence scores (top). Then we visualize the object and grasp logits outputted by the grasp policy of GSP (middle) and GSP-BC (bottom) respectively. Note that the object logits are obtained by the box-grasp mapping matrix described in Sec.~\ref{sec:or-cases}. We mark the grasp poses of the object with triangles, and the selected grasp of the policy with an arrow.}
  \label{fig:or-cases}
  \vspace{-0.6cm}
\end{figure*}
{\bf Case Studies.} 
\label{sec:or-cases}
Fig.~\ref{fig:or-cases} visualizes two cases for comparison of grasp policies. For each case, we present the object logits and grasp logits from the grasp policies of GSP and GSP-BC. Note that the object logits are derived from the grasp logits, as the policies directly predict the grasp logits. Specifically, we assign grasps to objects with 3D position distances to get a $M\times K$ box-grasp mapping matrix. Then, each object logit corresponds to the highest grasp logit of its assigned grasps. The first case involves the challenge of swap and selectivity, and there is an object matching error of the third goal object, which is initially occupied by a non-goal object. By learning with task-level rewards, GSP shows multimodal high logits for non-goal objects, and selects to grasp one of the non-goal objects {\it i.e.} orange. Instead, GSP-BC shows a concentrative high logit on the pink box, the last goal object. This is because behavior cloning guides the policy with unimodal distributions. Although the pink box matches with the highest confidence score (0.339) across all current objects, its goal is occupied by other objects. That is, it should be placed to the buffer. The second is a case with swap only. GSP predicts high logits for the fifth object and the last object, whose goal is all free for placement, and chooses to grasp the last one. However, GSP-BC grasps an object whose goal is occupied. This may be because the grasp policy of GSP-BC pays less attention to task-level efficiency, and behavior cloning brings bias of the training data into the policy.

\subsection{Real-world Experiments}
\begin{figure}[t]
  \centering
  \includegraphics[width=0.95\linewidth]{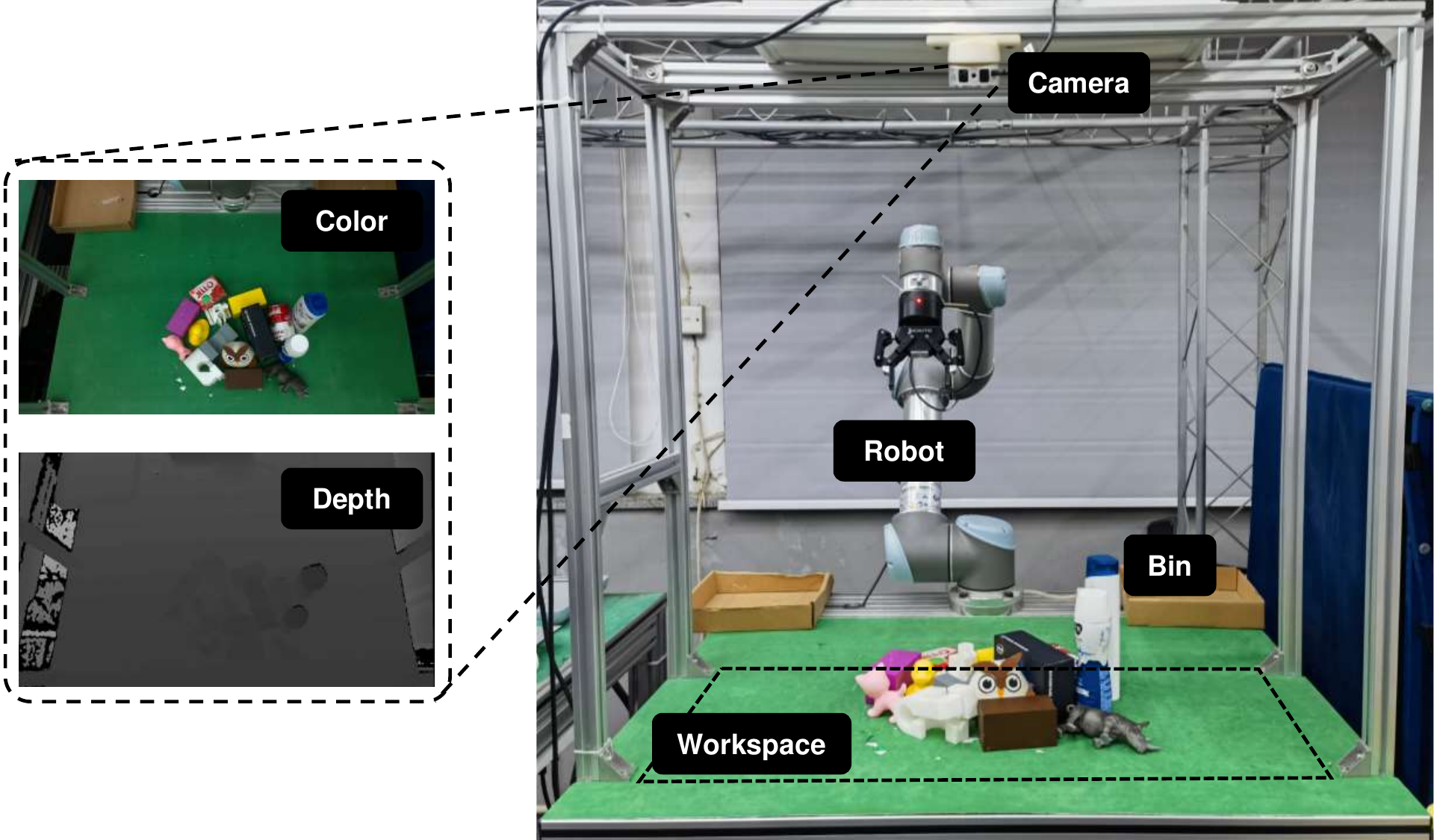}
  \vspace{-0.2cm}
  \caption{The real-world platform includes a UR5 robot arm with a ROBOTIQ-85 gripper. The bin next to the robot holds the non-goal objects. The Intel RealSense L515 camera captures RGB-D images of resolution 1280 $\times$ 720.}
  \label{fig:real-platform}
  \vspace{-0.4cm}
\end{figure}

\begin{figure}[t]
  \centering
  \includegraphics[width=\linewidth]{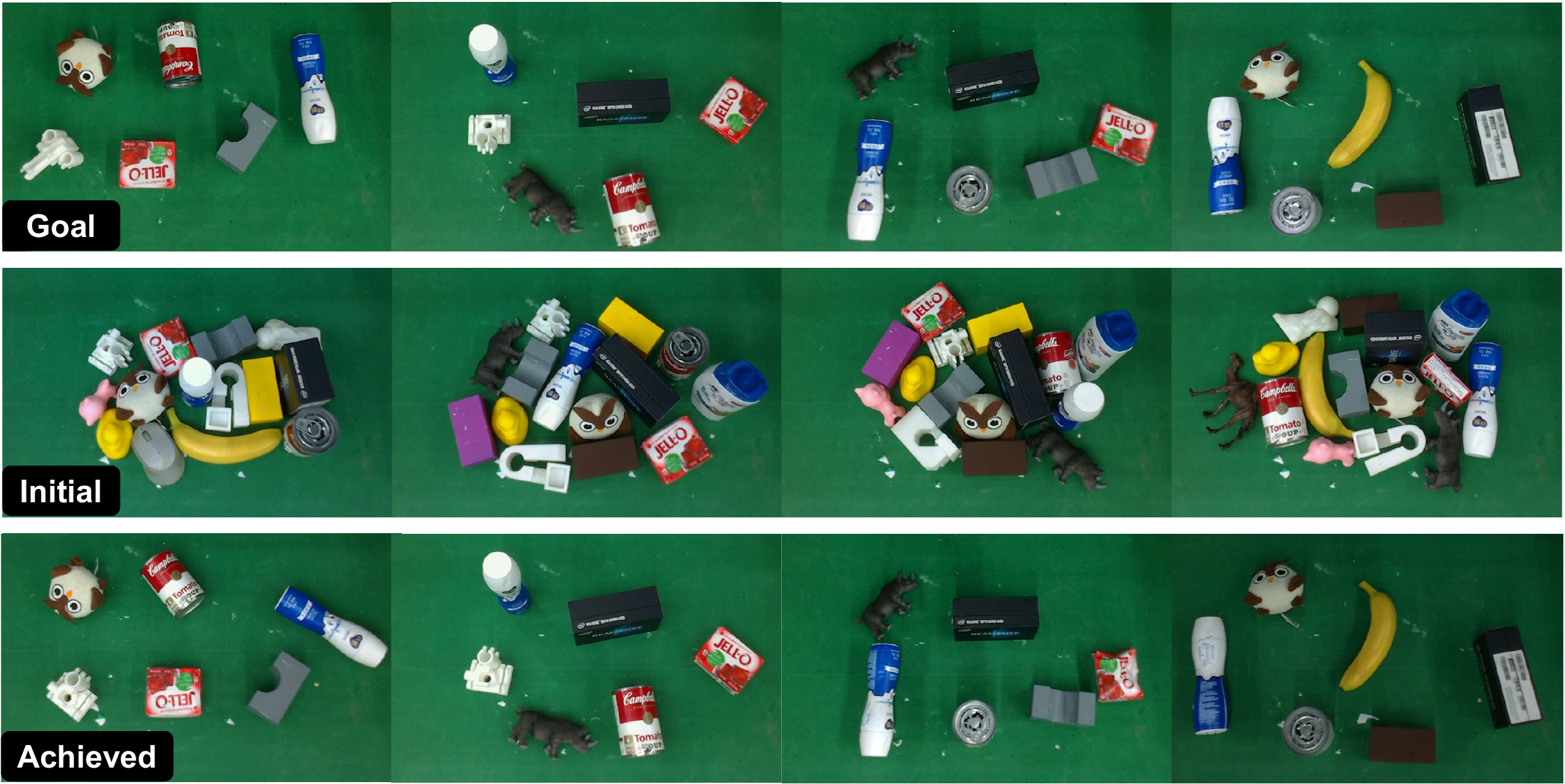}
  \vspace{-0.6cm}
  \caption{The real-world cases~(first two rows) and the corresponding achieved scenes by GSP~(the bottom row). Each case consists of 14$\sim$15 objects initially, and 6 objects in the goal scene. All of the scenes present the challenges of {6-DoF pose transformations}, swap, clutter and selectivity.}
  \label{fig:real-case}
  \vspace{-0.7cm}
\end{figure}

\begin{table}[t]
\caption{Real-world Results on Object Rearrangements}
\label{table:real-rearrangement}
\vspace{-0.1cm}
\centering
\resizebox{\linewidth}{!}{
\begin{tabular}{ccccc}
\hline
& Task Completion & Overall Planning Steps & Pos Error & ADD-S \\
\hline
\multicolumn{5}{c}{IFOR$^\dag$} \\
\hline
Case 1 & 60.0 & 27.6 $\pm$ 5.1 & 4.65 $\pm$ 0.14 & 2.01 $\pm$ 0.40\\
Case 2 & 40.0& 32.0 $\pm$ 4.9 & 4.37 $\pm$ 0.02 & 2.02 $\pm$ 0.50\\
Case 3 & 20.0& 36.0 $\pm$ 4.0 & 4.49 $\pm$ 0.00 & 0.95 $\pm$ 0.00\\
Case 4 & 20.0& 36.8 $\pm$ 3.2 & 4.28 $\pm$ 0.00 & 2.35 $\pm$ 0.00\\
Average & 35.0 & 33.1 $\pm$ 4.7 & 4.49 $\pm$ 0.17 & 1.91 $\pm$ 0.55\\
\hline
\multicolumn{5}{c}{GSP} \\
\hline
Case 1 & 60.0& 29.6 $\pm$ 4.4 & 4.09 $\pm$ 0.29 & 1.42 $\pm$ 0.19\\
Case 2 & 60.0& 28.8 $\pm$ 4.6 & 4.55 $\pm$ 0.30 & 2.28 $\pm$ 0.24\\
Case 3 & 80.0& 24.4 $\pm$ 3.9 & 2.76 $\pm$ 0.57 & 1.01 $\pm$ 0.26\\
Case 4 & 60.0& 29.2 $\pm$ 4.4 & 3.89 $\pm$ 0.28 & 1.82 $\pm$ 0.40\\
Average & {\bf 65.0} & {\bf 28.0 $\pm$ 4.5} & {\bf 3.74 $\pm$ 0.80} & {\bf 1.58 $\pm$ 0.56} \\
\hline
\end{tabular}}
\vspace{-0.7cm}
\end{table}

{{\bf Test Settings.}} In this section, we evaluate our system in real-world settings. Our real-world platform involves a UR5 robot arm with a ROBOTIQ-85 gripper, an Intel RealSense L515 capturing RGB-D images of resolution 1280$\times$720, and a bin next to the robot holding the non-goal objects~(shown in Fig.~\ref{fig:real-platform}). Object bounding boxes are generated by a pre-trained open-set detection model~\cite{zhou2023open}. The detection model is trained with data from GraspNet-1Billion with $mAP=70.70$ for known objects and $mAP=34.53$ for unknown objects. Our testing cases consist of 4 scenes, each of which has 14$\sim$15 objects initially, and 6 objects in the goal scene~(shown in Fig.~\ref{fig:real-case}). All the scenes present the challenges of {6-DoF pose transformations}, swap, clutter and selectivity. Also, each scene contains novel objects unseen during training to evaluate the generalization of the system. We compare our methods with the method {IFOR}$^{\dag}$ that shows a better performance than other baselines in our simulation experiments.

In real-world experiments, the image crop of the grasped object is screened out through the depth information and the distance to the gripper tip. If active perception is needed, the robot moves to a pre-defined see pose as the initial pose of the following see actions.

{{\bf Comparisons to Baselines.}} We test 5 times for each scene, in total 20 times testing. While it is difficult to reproduce the precise initial scene in the real world, we try to duplicate the initial configurations with the same RGB-D image. Notably, both methods are transferred from simulation to the real world without any retraining. We report results on task completion rate, {average overall planning steps}, average position error, and average ADD-S metric in Table~\ref{table:real-rearrangement}. Note that in the real world, we extend the planning step limitation to $b=40$. Results show that our system is able to transfer to the real world with novel objects, which demonstrates the generalization capability. Instead, {IFOR$^{\dag}$} shows a task completion rate low at 35$\%$. For objects required to be rearranged with large {pose transformations}~({\it i.e.} the tomato soup can needs to be rearranged from lying flat to standing upright), {IFOR$^{\dag}$} may match wrong goal objects by one shot. This leads to frequent failure, as IFOR computes pose transformations based on these wrong matches during object and action planning. In contrast, our policy, while potentially starting with an inaccurate match, successfully corrects itself by actively rotating the grasped object. Note that CLIP-Adapter is trained on multi-view images of the tomato soup in simulation, thereby minimizing the likelihood of misclassifying this object as a non-goal item.

{In general, policies incur more planning steps compared to the simulation due to the failure modes in place. First, the inaccurate 6-DoF pose estimation may cause false object placement, calling for additional steps to move the object to its goal. If the estimated pose is of large error, the motion planning module may fail to plan a feasible trajectory. Then the in-hand object should be reset manually to continue the experiment. Such steps are included in the planning steps. Second, objects may need multiple times of reorientation to achieve the goals. This may also cause failed motion planning considering the environmental constraints, leading to more planning steps.}

{{\bf Example Sequence.}} An example sequence of GSP is shown in Fig.~\ref{fig:real-sequence}. By actively seeing the object in hand, our system can get high self-confident matching, and iteratively rearrange objects that are not in their goal regions. For example, in step 9, the gray block is rearranged to the goal region of the jello, and in step 16, it is grasped once more and placed to the correct goal region. Other achieved cases are shown in Fig.~\ref{fig:real-case}. 

{{\bf Failure Cases.}} Fig.~\ref{fig:failure-case} shows two failure cases. In the first case, the initial scene comprises two gray blocks positioned horizontally, whereas the desired outcome involves only one gray block, arranged vertically. However, both gray blocks in the initial scene adopting an upright orientation can display the state in the goal scene, posing a challenge for the policy to identify the correct goal object. Thus, after one of the grey blocks is placed at the goal region, another will be always placed in the buffer. The second case includes two similar novel objects both in the initial and goal scenes, which are hard to distinguish by the policy, leading to misclassifying the objects~(one of them is misclassified as non-goal.) Additionally, if the object is rearranged to a totally different view, the policy may struggle to get an accurate matching without finetuning on multi-view images of that object.

\begin{figure}[t]
  \centering
  \includegraphics[width=\linewidth]{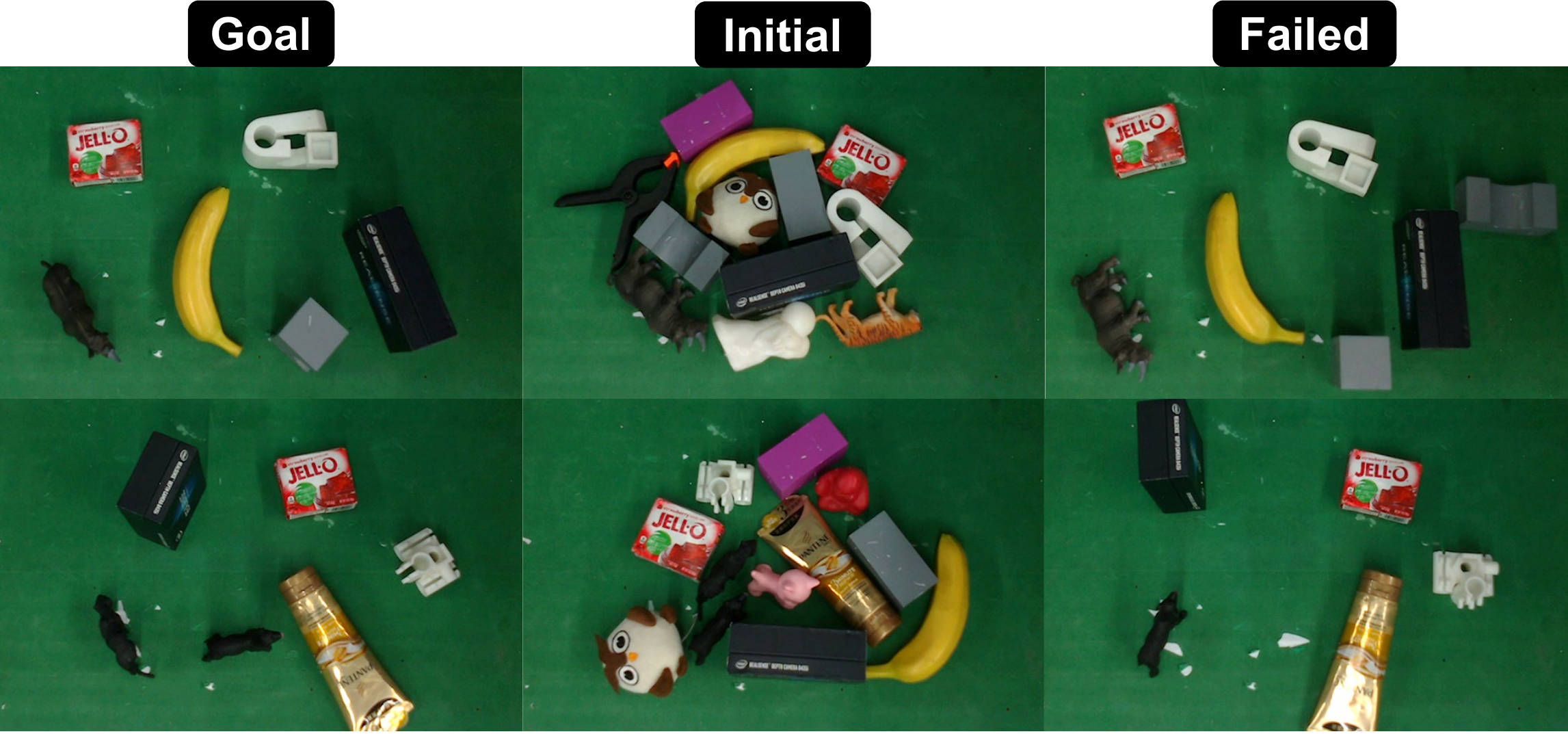}
  \vspace{-0.6cm}
  \caption{Typical failure cases. Both of the failure cases contain similar objects that are hard to distinguish.}
  \label{fig:failure-case}
  \vspace{-0.7cm}
\end{figure}

\begin{figure*}[t]
  \centering
  \includegraphics[width=\linewidth]{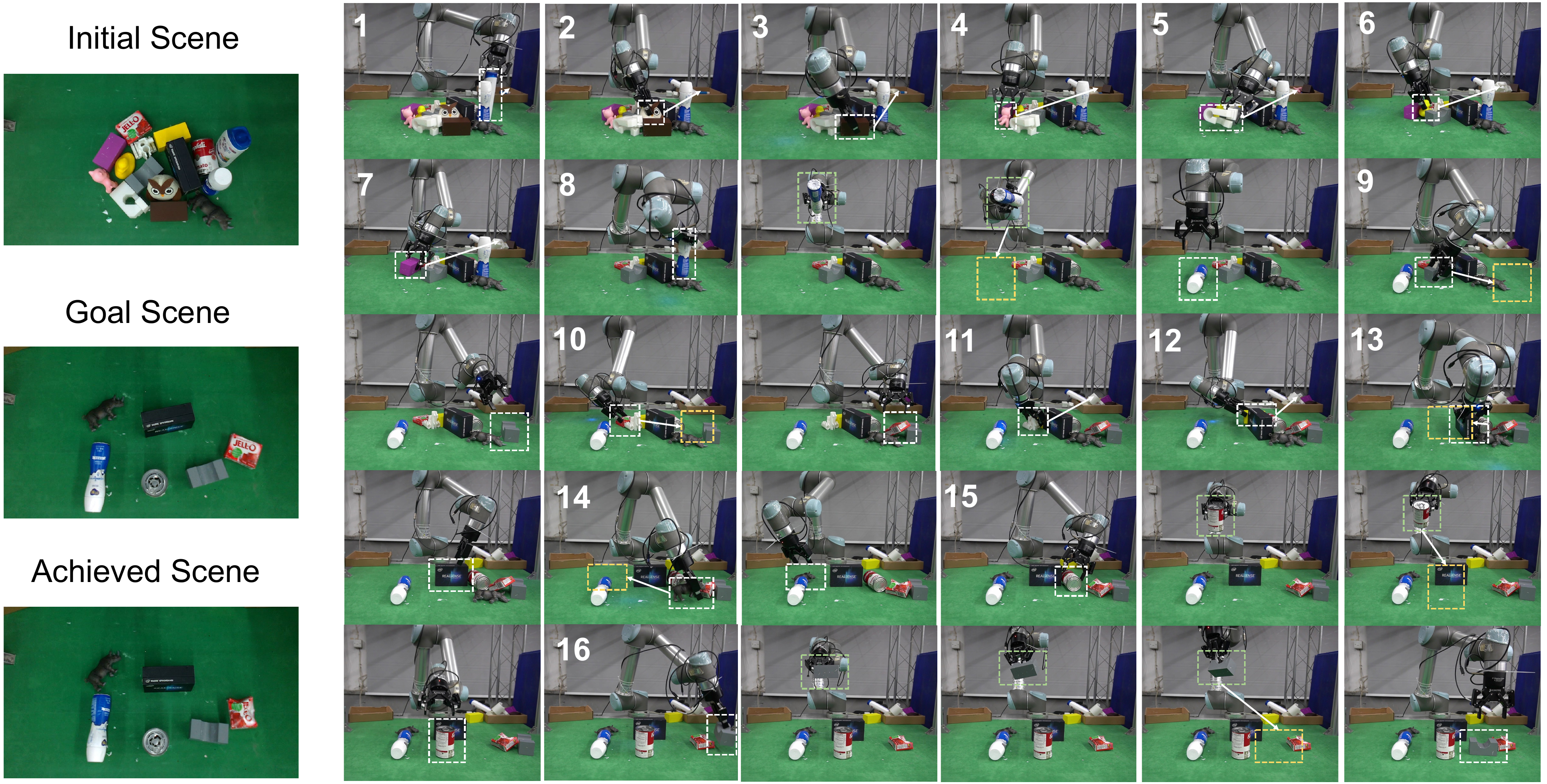}
  \vspace{-0.7cm}
  \caption{An example sequence of real-world object rearrangement, where the robot is faced with many everyday objects (clutter), only a subset objects in the goal scene (selectivity), occupied goal locations of some objects (swap) and {6-DoF pose transformations}. White boxes mark the manipulated objects, green boxes highlight the active {perception} process, white arrows show the moving directions, and yellow boxes show the planned places. Each number refers to the manipulation process of one object. Our method allows the robot to plan a sequence of actions to grasp each object, rotate it for active {perception} if necessary, and place it to the desired goal pose or outside.}
  \label{fig:real-sequence}
  \vspace{-0.7cm}
\end{figure*}

\section{Conclusion}
In this work, we study the task of unknown object rearrangement with swap, clutter, selectivity and {6-DoF pose transformations}. 
We propose GSP, a dual-loop system with the decoupled structure as prior for efficient unknown object rearrangement. We introduce the skill of {see} for self-confident object matching as the inner loop, which improves {the in-hand object matching}. Then grasp and place planning serves as the outer loop. The place policy is rule-based, while the grasp policy is learned to be aware of noisy object matching and grasp capability, guided by task-level rewards. We evaluate our method in simulation and the real world. Results show that GSP can conduct unknown object rearrangement with higher success rates using {fewer} steps. 

{
{\bf Future Work.} To further employ GSP in industrial applications, there are some remaining problems. For example, for 6-DoF transformation tasks, objects may need to be reoriented several times to achieve their target poses. In future work, we plan to make efforts on this problem for unknown objects. In addition, our system may struggle with stacked goal configurations due to the assumption of the goal scene. The assumption can be relieved by extending the goal representation to point clouds. It is worth noting that our theoretical analysis can be extended beyond the 2D tabletop setting. This is because our problem formulation is based on the feasibility of the $M$-to-$N$ object matching between current and goal objects, as well as the 1-to-$N$ object matching between the in-hand object and goal objects. Other settings, such as 3D tabletop and room-level environments, only affect the degree of uncertainty in object matching. We believe that extending GSP to 3D tabletop setting, and further to the room-level environment {for mobile manipulation}, is a meaningful direction for future work. }




\bibliographystyle{IEEEtran}
\bibliography{IEEEabrv,ref}


{
\newpage
\appendix
\subsection{Proof of Theorem 1}
\label{sec:appendix-a}

\begin{figure}[t]
  \centering
  \includegraphics[width=\linewidth]{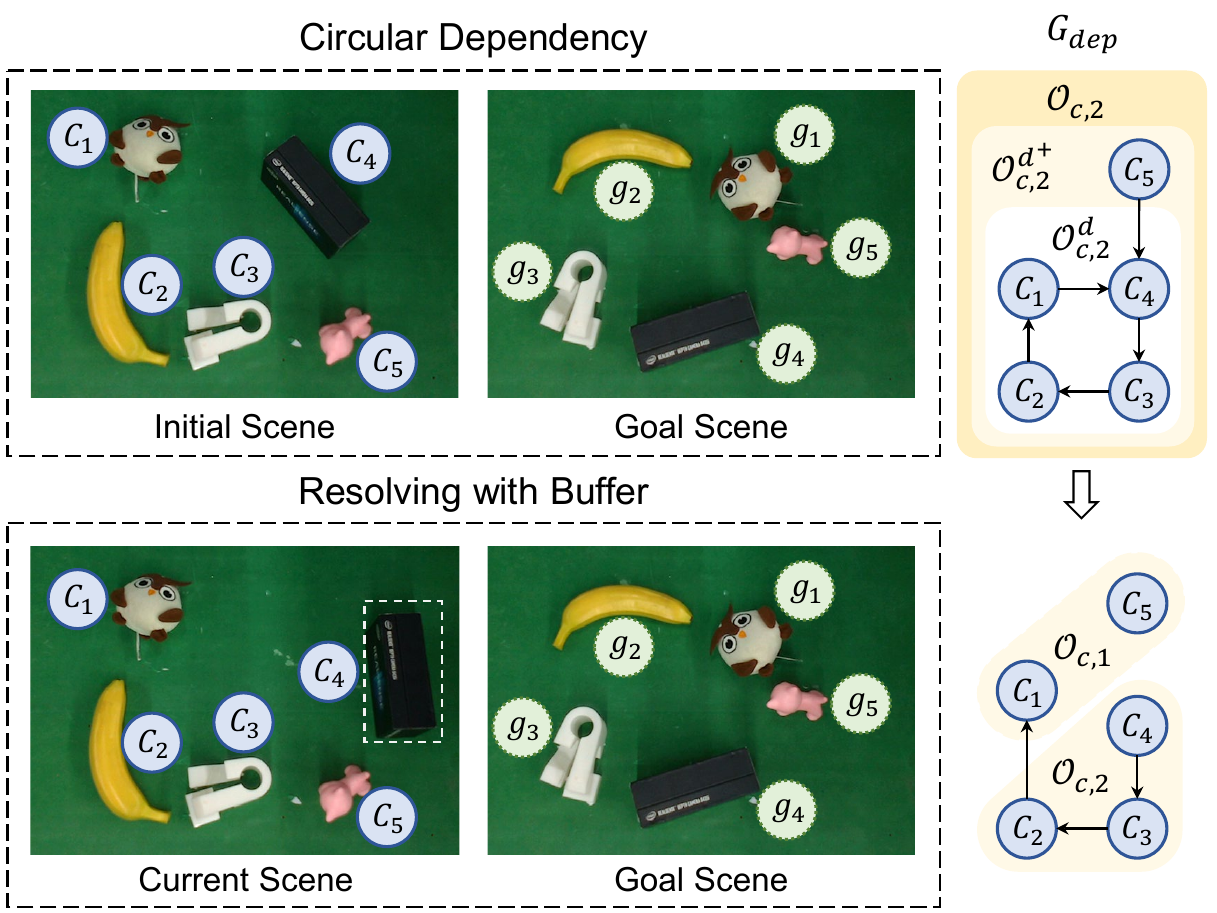}
  \vspace{-0.6cm}
  \caption{An example to illustrate circular dependency, the corresponding dependency digraph, and buffer. After moving $c_4$ to the buffer~(marked with the white box), the circular dependency breaks.}
  \label{fig:dependency-graph}
  \vspace{-0.6cm}
\end{figure}

{\bf Circular Dependency.} It is possible to have circular dependencies~\cite{gao2021minimizing,gao2022fast}. An example is shown in Fig.~\ref{fig:dependency-graph}. The objects cannot be moved to the goals until the circular dependency is broken. Formally, these objects form cycles in the graph $G_{dep}$, which we denote the set of circular dependent objects as $\mathcal{O}_{c,2}^d$:
\begin{equation}
\begin{aligned}
\mathcal{O}_{c,2}^d:\{o_c^{i_0}\big|&\exists o_c^{i_1},o_c^{i_2},...,o_c^{i_n}\in\mathcal{O}_c,n\geq1, \\
&(o_c^{i_0},o_c^{i_1}),(o_c^{i_1}, o_c^{i_2}),...,(o_c^{i_n}, o_c^{i_0})\in A_{dep}\}\nonumber
\end{aligned}
\end{equation}
Furthermore, if an object has a path directed to $\mathcal{O}_{c,2}^d$, it also cannot be moved to the goals before the circular dependency breaks. The set including this kind of objects, along with objects in $\mathcal{O}_{c,2}^d$, is denoted as $\mathcal{O}_{c,2}^{d^+}$: 
\begin{equation}
\mathcal{O}_{c,2}^{d^+}:\mathcal{O}_{c,2}^d \cup \left\{o_c^{i_0}\big|\exists o_c^{i_1}\in\mathcal{O}_{c,2}^d, (o_c^{i_0}, o_c^{i_1})\in A_{dep}\right\}\nonumber
\end{equation}
Then we have $\mathcal{O}_{c,2}^d\subseteq\mathcal{O}_{c,2}^{d^+}\subseteq\mathcal{O}_{c,2}$. Once existing circular dependencies, some objects must be temporarily moved to the intermediate places to break the dependencies. These intermediate places are defined as {buffers}. A buffer can be a place that is currently not occupied and does not overlap with all goal regions. Fig.~\ref{fig:dependency-graph} shows an illustration of the buffer. We follow the assumption in \cite{han2018complexity} that the capacity of the buffer is infinite, which is practical since the workspace in daily life is large enough relative to the space occupied by the objects.

{\bf Optimal Number of Steps.} In \cite{han2018complexity}, the problem of resolving objects within circular dependency mirrors the Feedback Vertex Set~(FVS) problem, which is a set of vertices whose removal makes $G_{dep}$ acyclic. After resolution, all objects can be placed to goals without buffers. Then we have the optimal number of steps for object rearrangement proved in \cite{han2018complexity}:
\begin{lemma}
    Given the dependency graph $G_{dep}$ of a tabletop object rearrangement problem of $M$ objects, the minimum steps to resolve objects in $\mathcal{O}_{c,2}^d$ equals the size of the minimum FVS of $G_{dep}$. Let this size be $|B|$, the minimum total pick-n-place steps to achieve the goal configuration is $M+|B|$.
\end{lemma}

Based on the lemma, we have the proof of Theorem 1:
\begin{proof}
For each object in the current scene, it belongs to $\mathcal{O}_{c,1}$ or $\mathcal{O}_{c,2}^{d^+}$ or $\mathcal{O}_{c,2}\setminus\mathcal{O}_{c,2}^{d^+}$. We denote $0/1$ as the existence of objects in these three sets, {\it e.g.} $\{111\}$ means the current scene contains objects in all three sets, resulting in $2^3$ situations. Notably, there are paths from objects in $\mathcal{O}_{c,2}\setminus\mathcal{O}_{c,2}^{d^+}$ to those in $\mathcal{O}_{c,1}$. Thus, the existence of objects in $\mathcal{O}_{c,2}\setminus\mathcal{O}_{c,2}^{d^+}$ indicates the existence of objects in $\mathcal{O}_{c,1}$. In addition, the empty scene is not considered. Therefore, there are at most 5 cases for the current scene: $\{010\}, \{100\}, \{101\}, \{110\}, \{111\}$. Furthermore, we merge case 2 and case 3 as $\{10\cdot\}$, {\it i.e.} the current scene contains objects in $\mathcal{O}_{c,1}$, but does not contain objects in $\mathcal{O}_{c,2}^{d^+}$, and merge case 4 and case 5 as $\{11\cdot\}$, {\it i.e.} the current scene contains both objects in $\mathcal{O}_{c,1}$ and $\mathcal{O}_{c,2}^{d^+}$. 

Let $\mathcal{F}^\pi(\mathcal{O}_c|\mathcal{O}_g)$ be the total steps of a policy $\pi$ to reach the goal configuration $\mathcal{O}_g$ from $\mathcal{O}_c$. Following {\bf Lemma 1}, an optimal policy $\pi^*$ has $\mathcal{F}^{\pi^\star}(\mathcal{O}_c|\mathcal{O}_g)=M+|B|$. To prove $\pi^0$ is optimal, we show that $\pi^0$ can achieve the minimum total pick-n-place steps in all cases:

{\it 1) Existence of both $\mathcal{O}_{c,1}$ and $\mathcal{O}_{c,2}^{d^+}$, $\{11\cdot\}$.}

In this case, there exist objects $o_c^i\in\mathcal{O}_{c,1}$, the policy $\pi^0$ grasps objects $o_c^{i} \in \mathcal{O}_{c,1}$ and directly places them to their goals. Note that moving an object $o_c^{i}\in\mathcal{O}_{c,1}$ may bring new objects whose goals are initially occupied by $o_c^{i}$ into the set of $\mathcal{O}_{c,1}$. Such objects are then moved to their goals. That is, objects in $\mathcal{O}_{c,1}$ are iteratively moved to their goals until there is no object in $\mathcal{O}_{c,1}$. Let the number of these objects be $M^0$, then the number of steps to move them to the goals is $M^0$.

After these objects are all moved to their goals, there must remain only objects in $\mathcal{O}_{c,2}^{d^+}$. To prove this, we assume there is an object $o_c^{i_0}$ that does not belong to $\mathcal{O}_{c,2}^{d^+}$ after moving all objects in $\mathcal{O}_{c,1}$. That is, the reachable subgraph of $o_c^{i_0}$ is acyclic. So, there must be another object $o_c^{i_1}$ in the reachable subgraph of $o_c^{i_0}$, that does not depend on any other objects. As a result, $o_c^{i_1} \in \mathcal{O}_{c,1}$, which is in conflict with the assumption. 

For objects in $\mathcal{O}_{c,2}^{d^+}$, the policy $\pi^0$ first resolves them as $\pi_{\mathcal{G}}^{fvs}$, which costs $|B|$ times of place in the buffer. After resolving the circular dependency by $\pi_{\mathcal{G}}^{fvs}$, all the remaining $M-M^0$ objects can be moved to their goals without additional places in buffers. 
Therefore, the number of steps to move these objects to the goals is $M-M^0+|B|$. 

Finally, the total steps of the policy equals $M+|B|$.

{\it 2) Existence of $\mathcal{O}_{c,1}$, absence of $\mathcal{O}_{c,2}^{d^+}$, $\{10\cdot\}$.}

In this case, there is no circular dependency. That is, all the objects can be moved to the goals without buffers {\it i.e.} $|B|=0$. 
Therefore, $\mathcal{F}^{\pi^0}(\mathcal{O}_c|\mathcal{O}_g)=M=M+|B|$.

{\it 3) Existence of only $\mathcal{O}_{c,2}^{d^+}$, $\{010\}$.}

In this case, the policy $\pi^0$ first resolves circular dependency as $\pi_{\mathcal{G}}^{fvs}$, which costs $|B|$ times of place in the buffer. Then, all the objects can be moved to their goals without additional places in buffers. Therefore, the number of steps to move all the objects to the goals is $M+|B|$. 

Overall, we have $\mathcal{F}^{\pi^0}\!(\mathcal{O}_c|\mathcal{O}_g)\!=\!\mathcal{F}^{\pi^\star}\!(\mathcal{O}_c|\mathcal{O}_g)\!=\!M+|B|$ in all possible cases, demonstrating that the policy $\pi^0$ is optimal.

\end{proof}

\subsection{Proof of Theorem 2}
\label{sec:appendix-b}
\begin{proof}
It has been stated in Eqn.~\ref{eqn-pi0-place} that with ideal perception, the optimal place policy $\bar{\pi}_{\mathcal{P}}^0$ follows a rule conditioned on the object matching derived from the grasp policy. With perception noise, the derived object matching may be wrong. 

Given two current objects $o_c^{i_0}$, $o_c^{i_1}$ and their ground-truth matched objects $o_g^{j_0}$, $o_g^{j_1}$, assume that the object matching derived from $\pi_{\mathcal{G}}^0$ wrongly pairs $o_c^{i_0}$ to $o_g^{j_1}$. Instead, by improving the in-hand object matching independently, $\bar{\pi}_{\mathcal{P}}^1$ can get the correct matching of $o_c^{i_0}$ to $o_g^{j_0}$. 

Now let $o_c^{i_0}$ be the in-hand object. If $o_g^{j_0}$ and $o_g^{j_1}$ are both occupied, $\pi^0$ and $\pi^1$ will both move $o_c^{i_0}$ to the buffer, then the wrong matching does not affect their pick-n-place steps. However, if $o_g^{j_0}$ is not occupied, $\pi^1$ will move $o_c^{i_0}$ to the goal. Instead, $\pi^0$ will move $o_c^{i_0}$ to the buffer or the goal of $o_c^{i_1}$, which calls for additional steps to move $o_c^{i_0}$ to its goal, or brings additional occupancy of the goal of  $o_c^{i_1}$. Thus, for the rearrangement of $o_c^{i_0}$, $\pi^0$ costs more steps than $\pi^1$.

Notably, each time there is only one in-hand object, and the effect on total pick-n-place steps can be accumulated if there is more than one wrong matching of the in-hand objects. Therefore, $\mathcal{F}^{\pi^0}\!(\mathcal{O}_c|\mathcal{O}_g)\!\geq\!\mathcal{F}^{\pi^1}\!(\mathcal{O}_c|\mathcal{O}_g)$, indicating that the improvement of perception of the place policy is {valuable} to reduce the total pick-n-place steps, {\it i.e.} improve task-level performance.

\end{proof}

\subsection{More Implementation Details of See Policy}
\label{sec:appendix-c}
{\bf RL Algorithm Details.} We train the see policy with SAC~\cite{haarnoja2018soft,christodoulou2019soft}. The temperature parameter $\alpha$ is initialized as 0.2 with automatic entropy tuning. All the networks are trained with Adam optimizer using fixed learning rates $3\times 10^{-4}$. The future discount $\gamma$ is set as a constant at 0.99. Note that during the training stage, the network parameters of the finetuned CLIP are fixed.

{\bf Network Architecture.} For finetuned CLIP, we use the CLIP-Adapter~\cite{gao2021clip} containing a 3-layer MLP with sizes [512, 128, 512]. For see policy, the policy and critic share the same state encoder with separate MLPs to output action and Q value respectively. To be specific, for policy, we employ ResNet50~\cite{he2016deep} to encode the state $\Delta f$, followed by MLPs to generate the mean and standard deviation of a normal Gaussian distribution. Subsequently, the action $a_{\mathcal{S}}$ is sampled from the distribution. Each angle of delta orientation is limited by the range of $[-1.57, 1.57]$. The MLPs of policy and critic are with sizes [1024, 1024, 3] and [1024, 1024, 1] respectively. 

\subsection{More Implementation Details of Grasp Policy}
\label{sec:appendix-d}

{\bf RL Algorithm Details.} The reinforcement learning algorithm used to train our networks is discrete SAC~\cite{haarnoja2018soft, christodoulou2019soft}. The temperature parameter $\alpha$ is initialized as 0.2 with automatic entropy tuning. All the networks are trained with Adam optimizer using fixed learning rates $3\times 10^{-4}$. Our future discount $\gamma$ is set as a constant at 0.99. Note that during the training stage, the network parameters of the finetuned CLIP and graspnet are fixed.

{\bf Network Architecture.} We adopt the transformer architecture of the text encoder in \cite{radford2021learning}, and conduct multi-head attention with different query, key and value. Parameters of each cross transformer include width of 512, head of 8 and layer of 1. The sizes of hidden layers of $\mathrm{MLP_1}$, $\mathrm{MLP_2}$ and $\mathrm{MLP_3}$ are [256, 512, 512], [256, 512, 512] and [512, 256, 1]. 

\subsection{Case Visualization}
\label{sec:appendix-e}

Fig.~\ref{fig:as-case-vis} shows example testing cases of see with different categories. Fig.~\ref{fig:or-case-vis} visualizes example testing cases of different categories in the tabletop rearrangement. 

\begin{figure}[t]
    \centering
    \includegraphics[width=\linewidth]{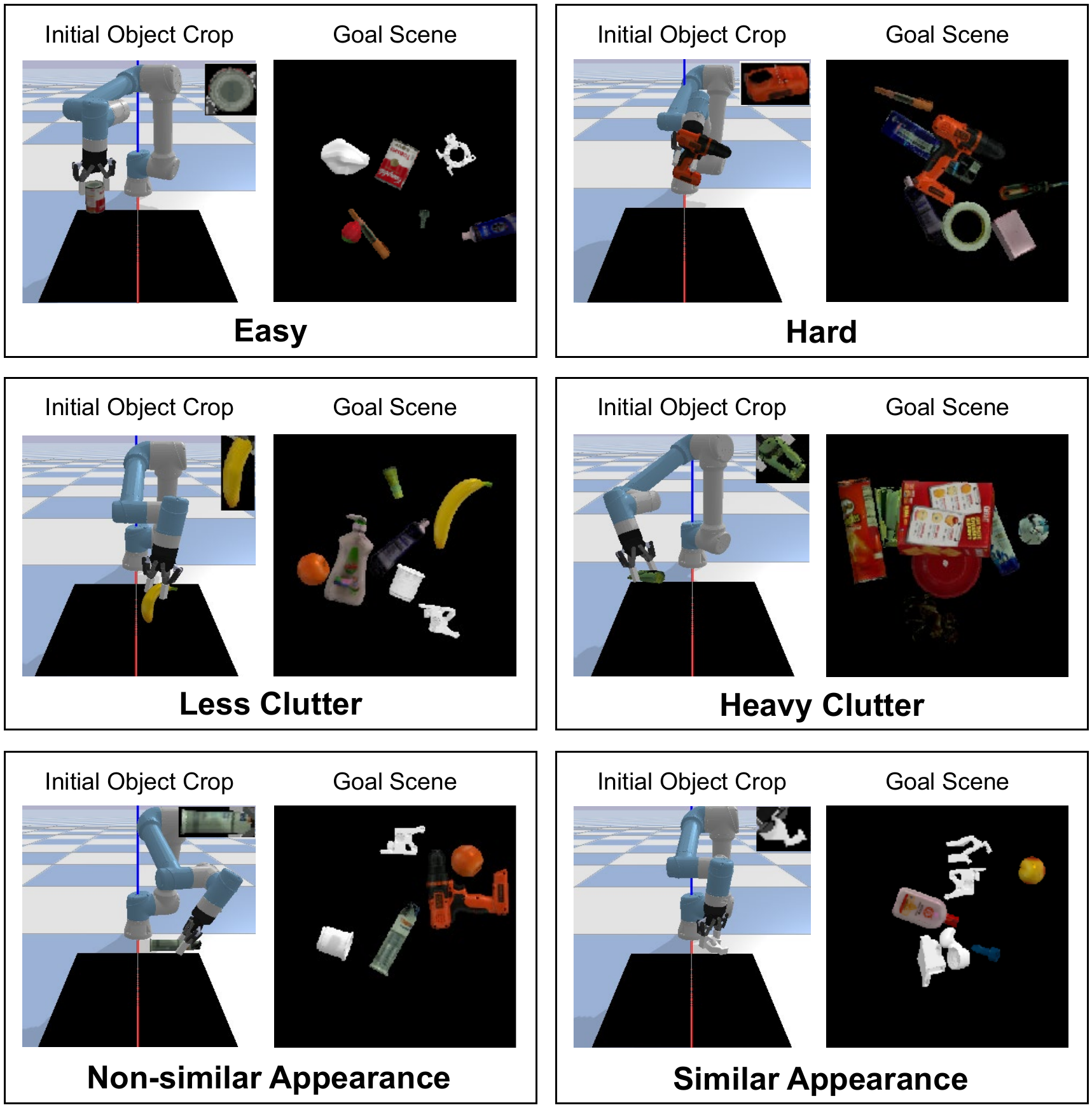}
    \vspace{-0.6cm}
    \caption{Example testing cases of different categories in see.}
    \label{fig:as-case-vis}
    \vspace{-0.3cm}
\end{figure}

\begin{figure}[t]
    \centering
    \includegraphics[width=\linewidth]{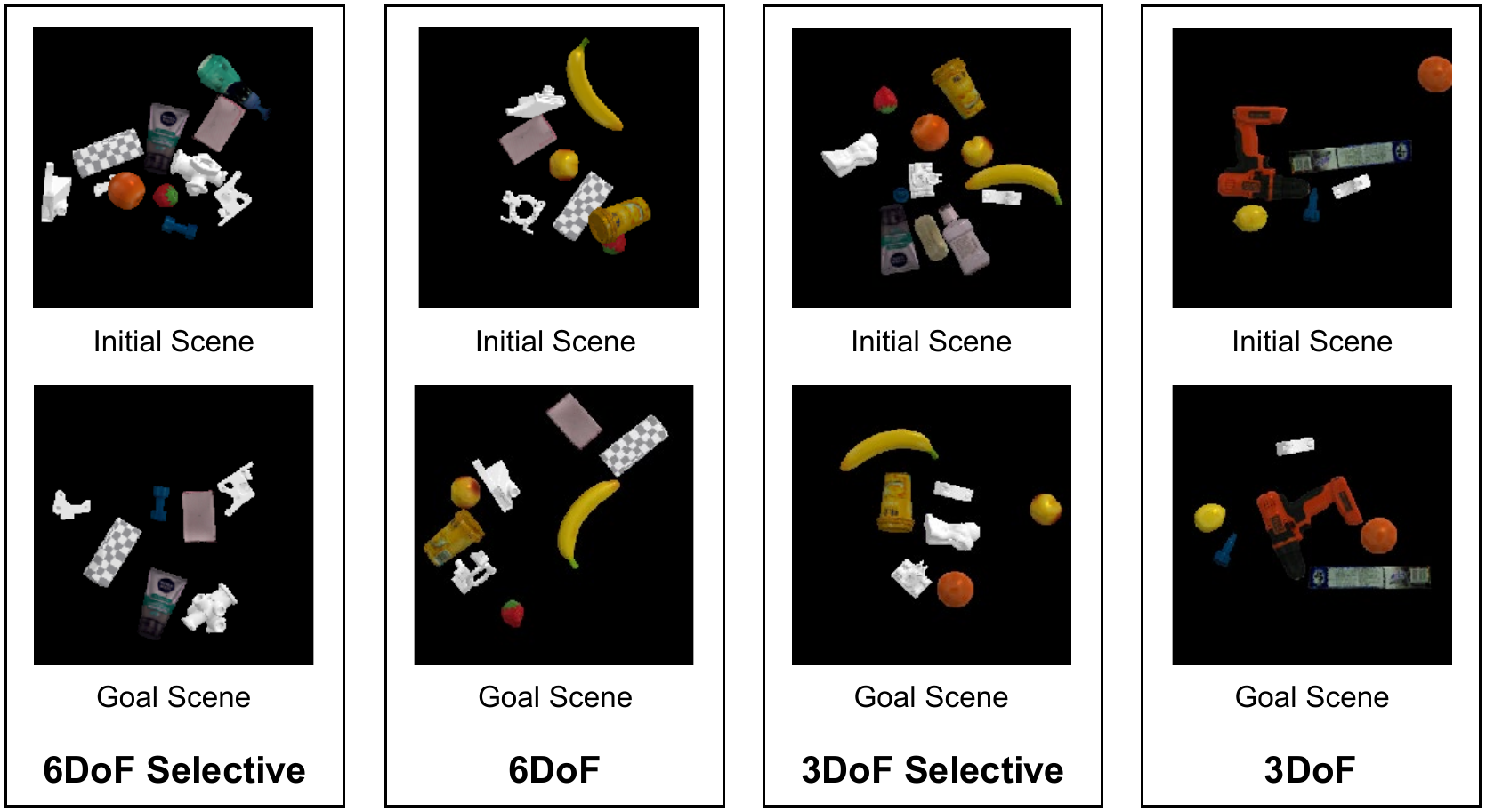}
    \vspace{-0.6cm}
    \caption{Example testing cases of different categories in object rearrangement.}
    \label{fig:or-case-vis}
    \vspace{-0.6cm}
\end{figure}

\subsection{Generalization on Perception Noise}
\label{sec:appendix-f}
To demonstrate the generalization of our method in noisy perception, we conduct experiments with unseen camera angle and table background in Table~\ref{table:simulated-diverse-noise}. The original camera angle is nearly top-down, while the newly tested camera overlooks the workspace from the front at a 36$\degree$ downward angle. Additionally, we conduct experiments with an unseen wooden table background. To overcome worse object matching, the maximum see steps is extended from $5$ to $10$. Environment settings of the unseen camera view and table background are visualized in Fig.~\ref{fig:or-case-noise}. Results show that our method can generalize to the unseen camera view and the table background. For the unseen table background, our policy achieves similar completion rates with more planning steps compared to that with the original background. For the unseen camera angle, there are performance drops both in completion rates and planning steps.

For the unseen table background, we take the object-centric image crops as input, thereby the textured background slightly affects the $M$-to-$N$ object matching. Accordingly, the performance of the grasp policy is hindered to some extent. Therefore, the policy needs to spend more planning steps. Additionally, by extending the maximum see steps, the see policy reduces the background influence on the in-hand object matching, prompting the whole policy to reach similar completion rates as that with the original background.

Our current model can adapt to unseen camera views by projecting the camera images into heightmaps as described in \cite{zeng2021transporter} ({\it i.e.} the top-down workspace image). However, the unseen camera view simultaneously affects the current and goal scenes. Even with the object-centric image crops, the impact is still significant. Due to the affected $M$-to-$N$ object matching, the performance of the grasp policy is hindered, calling for more planning steps. Due to the affected in-hand object matching, even extending the maximum see steps, the improvement of active perception is inapparent. Therefore, compared to the results of the original camera view, there remains a 5$\%$$\sim$10$\%$ margin in completion rates.

\begin{figure}[t]
    \centering
    \includegraphics[width=\linewidth]{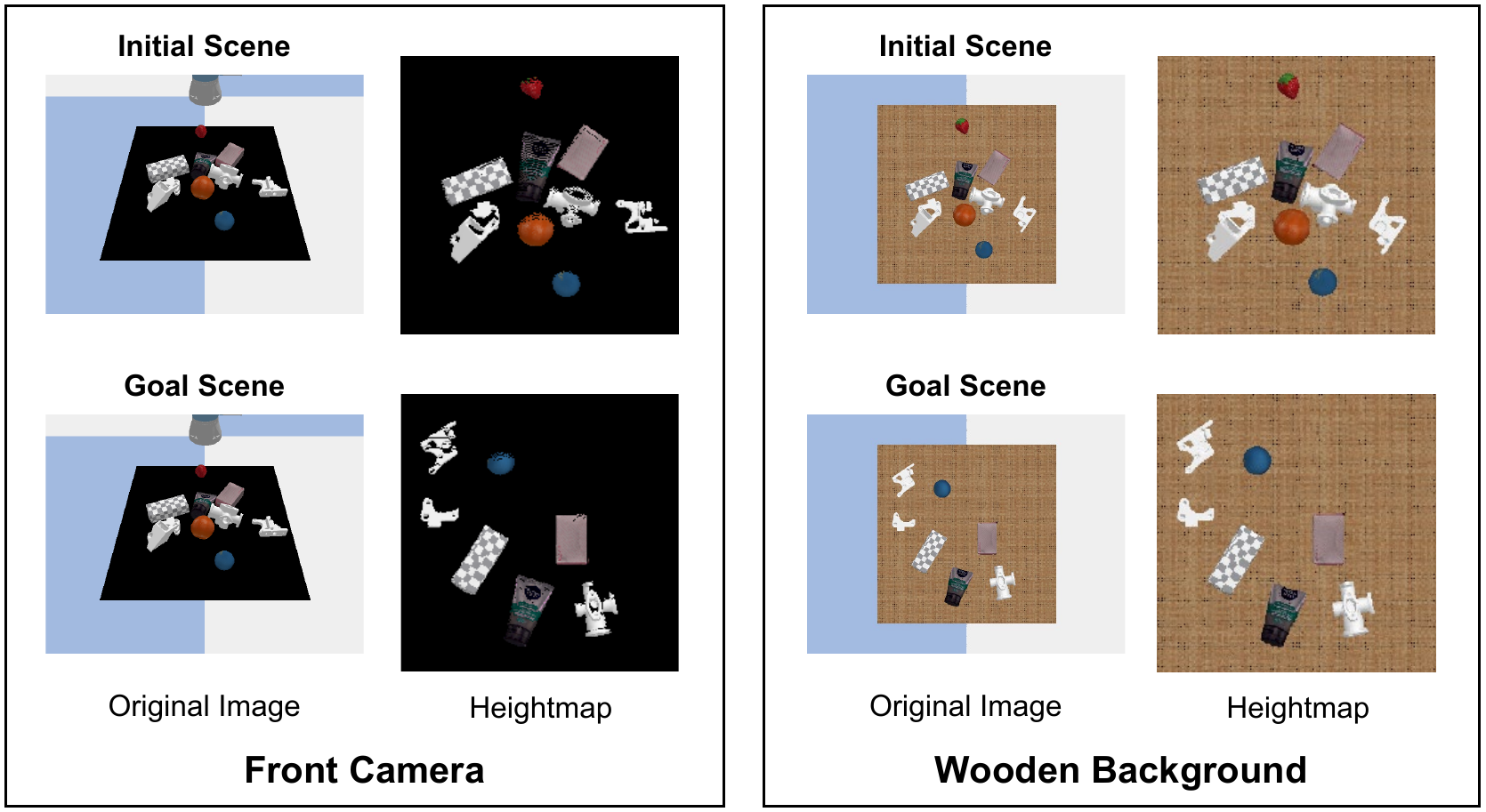}
    \vspace{-0.6cm}
    \caption{Example testing cases of generalization on perception noise in object rearrangement.}
    \label{fig:or-case-noise}
    \vspace{-0.6cm}
\end{figure}

\begin{table*}[t]
\caption{Simulation Results of Object Rearrangements on Generalization of Perception Noise\\(Maximum See Step=10)}
\label{table:simulated-diverse-noise}
\vspace{-0.2cm}
\begin{center}
\begin{threeparttable}
\begin{tabular}{ccccccccccccc}
\hline Environment & Rotation & \makebox[0.05\textwidth][c]{Selectivity}& Init. \#obj. & Goal \#obj. & \multicolumn{2}{c}{Task Completion} & \multicolumn{2}{c}{Completion Planning Steps} & \multicolumn{2}{c}{Overall Planning Steps}\\
\hline 
 & & & & & seen & unseen & seen & unseen & seen & unseen \\
\hline 
\multirow{4}{*}{Front Camera} & 3-DoF & \XSolidBrush & 4-8 & 4-8 & 87.5 & 70.0 & 19.6$\pm$3.5 & 22.5$\pm$3.2 & 20.9$\pm$3.7& 24.8$\pm$4.7\\
& 3-DoF & \Checkmark & 9-13 & 4-8 & 83.3 & 66.7 & 18.6$\pm$3.7 & 18.3$\pm$0.5 & 20.5$\pm$3.9 & 22.2$\pm$2.8 \\
& 6-DoF & \XSolidBrush & 4-8 & 4-8 & 80.0 & 66.7 & 21.8$\pm$3.9 & 19.2$\pm$1.7 & 23.4$\pm$3.8 & 22.8$\pm$3.5\\
& 6-DoF & \Checkmark & 9-13 & 4-8 & 75.0 & 66.7 & 19.6$\pm$2.5 & 25.6$\pm$1.5 & {22.2$\pm$3.5} & 27.1$\pm$1.8\\
\hline
\multirow{4}{*}{Wooden Background} & 3-DoF & \XSolidBrush & 4-8 & 4-8 & 92.0 & 66.7 & 15.3$\pm$2.0 & 12.2$\pm$1.1 & 16.5$\pm$4.3 & 18.1$\pm$4.8\\
& 3-DoF & \Checkmark & 9-13 & 4-8 & 91.7 & 66.7 & 14.4$\pm$0.8 & 18.3$\pm$1.5 & 15.7$\pm$1.5 & 22.2$\pm$3.5\\
& 6-DoF & \XSolidBrush & 4-8 & 4-8 & 90.0 & 66.7 & 12.1$\pm$1.8 & 13.0$\pm$1.8 & 13.9$\pm$1.9 & 18.7$\pm$4.7\\
& 6-DoF & \Checkmark & 9-13 & 4-8 & 84.6 & 66.7 & 20.0$\pm$2.3 & 20.0$\pm$1.3 & 21.5$\pm$3.7 & 23.3$\pm$4.5\\
\hline
\end{tabular}
\begin{tablenotes}
\scriptsize
\item * All cases are with challenges of clutter and swap.
\end{tablenotes}
\end{threeparttable}
\end{center}
\vspace{-0.7cm}
\end{table*}

}

\end{document}